\documentclass[letter,11pt]{article}
\usepackage{fullpage} 

\usepackage{color}
\usepackage{xcolor}
\usepackage{amssymb}
\usepackage{amsmath}
\usepackage{comment}
\usepackage{commath}
\usepackage{enumitem}
\usepackage{mathtools}
\usepackage{multirow}

\usepackage[algo2e]{algorithm2e}
\usepackage{algorithm}
\usepackage{nicefrac}
\usepackage{amsfonts}
\usepackage{commath}
\usepackage{esvect}
\usepackage{algorithmic}
\usepackage{natbib}
\usepackage{makecell}
\usepackage{cancel}
\newcommand{\one}{{\bf 1}}
\newcommand{\spa}{{\bf span}}
\newcommand{\spn}{{\bf span}}
\newcommand{\ep}{\mbox{ epoch }}
\newcommand{\Cone}{O(\frac{1}{p} H^2\log (T))}
\newcommand{\Ctwo}{O(\frac{1}{p^{4.5}}H^{9}S^2 A\log^2(SAT/\delta)\log^4(T))}
\newcommand{\jvalue}{\max(0,\lfloor \frac{\ell-K-1}{K} \rfloor)}

\newcommand{\Reg}{\text{Reg}}

\newcommand{\red}[1]{}
\newcommand{\redsug}[1]{}

\newcommand{\algoComments}[1]{\textcolor{blue}{#1}}

\newcommand{\newToCheck}[1]{\textcolor{blue}{#1}}

\newcommand{\blue}[1]{#1}
\renewcommand{\newToCheck}[1]{#1}

\newcommand{\Vinit}{\newToCheck{V^{\text{init}}}}
\newcommand{\gCthreeSymb}{\Gamma}
\newcommand{\gCthree}[1]{10#1 SAH^2b_0\log(2HSAT/\delta)} 

\newcommand{\VinitConstH}{\newToCheck{4(K+H)}}
\newcommand{\gExtraTerm}{4eK\zeta\gCthreeSymb_{K}}

\newcommand{\spaProj}{\overline{P}}
\newcommand{\Ex}{\mathbb{E}}
\newcommand{\calS}{\mathcal{S}}

\newcommand{\calA}{\mathcal{A}}
\newcommand{\constC}{2K(H+2)}
\newcommand{\constK}{\left\lceil \frac{2H}{p}\log(T) \right\rceil}
\newcommand{\alphaBonus}[1]{24H^*\sqrt{\frac{C\log(8SATH/\delta)}{#1+1}}}

\newcommand{\R}{\mathbb{R}}
\newcommand{\otilde}{\tilde{\mathrm{O}}}

\DeclareMathOperator*{\argmax}{arg\,max}

\newcommand{\regBoundO}{O\del{\frac{1}{p^3}H^5S\sqrt{AT\log(SAT/\delta)}\log^2(T)+ 
    \frac{1}{p^{4.5}}H^8 S^2A\log^2(SAT/\delta)\log^4(T)}}
\newcommand{\regBoundOtilde}{\Tilde{\mathrm{O}}\del{{H^5} S\sqrt{AT} + H^{8} S^2A}}

\usepackage{microtype}
\usepackage{graphicx}

\usepackage{booktabs} 
\usepackage{multirow}
\usepackage{comment}
\usepackage{thmtools} 
\usepackage{thm-restate}

\usepackage{hyperref}
\usepackage{lineno}

\usepackage{amsthm}

\usepackage[utf8]{inputenc} 
\usepackage[T1]{fontenc}    
\usepackage{hyperref}       
\usepackage{url}            
\usepackage{booktabs}       
\usepackage{amsfonts}       
\usepackage{nicefrac}       
\usepackage{microtype}      
\usepackage{xcolor}         
\usepackage{comment}

\usepackage{amsmath}
\usepackage{amssymb}
\usepackage{mathtools}
\usepackage{makecell}
\usepackage[capitalize,noabbrev]{cleveref}

\usepackage{float}

\usepackage{longtable}
\usepackage{comment}

\newcommand{\Lbar}{\overline{L}}

\usepackage{arydshln}

\usepackage{rotating}

\usepackage{array, multirow, bigdelim, booktabs} 
\usepackage{makecell}
\usepackage{mathtools}
\usepackage{commath}
\usepackage{xcolor,color}
\newcommand{\toberemoved}[1]{}
\newcommand{\algoname}{{Optimistic Q-learning}}
\newcommand{\Vbar}{\overline{V}}
\newcommand{\is}{{\cal I}}

\newtheorem{lemma}{Lemma}
\newtheorem{corollary}{Corollary}

\newtheorem{assumption}{Assumption}
\newtheorem{definition}{Definition}

\title{Optimistic Q-learning for average reward and episodic reinforcement learning}
\usepackage{times}

\author{
Priyank Agrawal\\
Columbia University \\
\texttt{pa2608@columbia.edu} \and
Shipra Agrawal\\
Columbia University\\
\texttt{sa3305@columbia.edu}}  

\begin{document}

\maketitle
\begin{abstract}
We present an optimistic Q-learning algorithm for regret minimization in average reward reinforcement learning under an additional assumption on the underlying MDP that for all policies, the time to visit some frequent state $s_0$ is finite and upper bounded by $H$, either in expectation or with constant probability. Our setting strictly generalizes the episodic setting and is significantly less restrictive than the assumption of bounded hitting time \textit{for all states} made by most previous literature on model-free algorithms in average reward settings. We demonstrate a regret bound of $\tilde{O}(H^5 S\sqrt{AT})$, where $S$ and $A$ are the numbers of states and actions, and $T$ is the horizon. A key technical novelty of our work is the introduction of an $\overline{L}$ operator defined as $\overline{L} v = \frac{1}{H} \sum_{h=1}^H L^h v$ where $L$ denotes the Bellman operator. Under the given assumption, we show that the $\overline{L}$ operator has a strict contraction (in span) even in the average-reward setting where the discount factor is $1$. Our algorithm design uses ideas from episodic  Q-learning to estimate and apply this operator iteratively. Thus, we provide a unified view of regret minimization in episodic and non-episodic settings, which may be of independent interest.\footnote{Accepted for presentation at the Conference on Learning Theory (COLT) 2025}
\end{abstract}

\section{Introduction}
Reinforcement Learning (RL) is a paradigm for optimizing the decisions of an agent that interacts sequentially with an unknown environment over time. RL algorithms must carefully balance \textit{exploration}, i.e., collecting more information,
and \textit{exploitation}, i.e., using the information collected so far to maximize immediate rewards. The mathematical model underlying any RL formulation is a Markov Decision Process (MDP). 
The algorithmic approaches for RL are typically categorized as model-based or model-free, depending on whether they learn the underlying MDP model explicitly or implicitly to learn the optimal decision policy.

Model-free approaches, such as Q-learning and policy gradient, directly learn the optimal values or policy. They have gained popularity in practice because of their simplicity and flexibility, and underlie most successful modern deep RL algorithms (e.g., DQN \citep{mnih2013playing}, DDQN \citep{vanhasselt2015deep}, A3C \citep{mnih2016asynchronous}, TRPO \citep{trpoShulman}). Technically, an algorithm is declared to be model-free if its space complexity is $o(S^2A)$, preferably $O(SA)$, with $S,A$ being the number of states and actions respectively \citep{li2021breaking}. At a more conceptual level though, the aim is to design algorithms that enjoy the structural simplicity and ease of integration of methods like Q-learning, value iteration, and policy iteration. 

The sample complexity and regret bounds for model-free approaches have often lagged behind the corresponding model-based approaches in the literature. The existing literature can be divided based on whether they consider settings with repeated episodes of fixed length $H$ (aka episodic setting), or average reward over the horizon $T$ under a single thread of experience with no restarts (aka average reward setting). For the \textit{episodic settings}, recent works \citep{jin2018q, li2021breaking} derived near-optimal regret upper bounds for variants of optimistic Q-learning based on the popular Upper-Confidence Bound (UCB) technique. In the \textit{average reward settings}, however, simple UCB-based extensions of Q-learning have only been able to achieve a $\tilde{O}(T^{2/3})$ regret bound \citep{wei2020model}. Recent $\sqrt{T}$ regret bounds with model-free algorithms do so by either introducing strong assumptions like worst-case hitting time and mixing times (e.g., \cite{wei2021learning}, \cite{hao2021adaptive}), or some elements of model-based learning. For example, the policy iteration and value-iteration based approaches in \cite{abbasi2019exploration} and \cite{hong2024provably} require storing historical state transition observations, thus essentially tracking empirical model estimates. \cite{zhang2023sharper} explicitly tracks pairwise state visit (i.e., state-action-next state) counters, thereby also tracking transition probability model estimates. 

Among these recent works on average reward setting, \cite{zhang2023sharper} is the first work, and to the best of our knowledge the only work, to provide a model-free algorithm with $\sqrt{T}$ regret bound under the weakest possible assumption of weakly communicating MDPs.  
Even though their algorithm tracks the model through pairwise state visit counters, a cleverly designed graph update routine ensures that the number of counters tracked is $O(SA)$ and not $S^2A$. Therefore, the algorithm technically qualifies as model-free. 
Although efficient, we contend that such an algorithm design is closer in spirit to a model-based approach and may not enjoy the same simplicity and flexibility that make model-free approaches like Q-learning attractive in the first place. Furthermore, the algorithm suffers from a high dependence on the size of the state space $S$ in the regret bound. 
(See Section \ref{sec: related work} for a more detailed comparison of our assumptions and regret bounds to the related work.) 
 
\paragraph{Our contributions.} In this paper, we present a simple optimistic Q-learning algorithm for regret minimization in tabular RL 
that is applicable in \textit{both average reward and episodic} settings. Our contributions include new modeling, algorithm design, and regret analysis techniques.
 
Our first and foremost contribution is a natural modeling assumption that generalizes the episodic and ergodic MDP settings and provides a more practically applicable formulation for many application domains. We consider the class of MDPs where there is an "upper 
 bound $H$ on the time to visit a frequent state $s_0$", either in expectation or with constant probability. The upper bound $H$ is assumed to hold under all feasible policies (stationary or non-stationary) and is known to the RL agent, although the identity of the frequent state $s_0$ may not be known. This assumption is naturally satisfied by the episodic settings since the terminal state is visited after every set of $H$ steps, and also by the ergodic MDP settings that assume bounded worst-case hitting time $H$ for {\it all states}.  
Furthermore, we demonstrate using several examples that it allows for significantly more modeling flexibility than these existing settings. 

As a classic example, consider the {\it queuing admission control} setting discussed in~\cite{puterman2014markov} (Example 8.4.1). Here, state is the number of jobs in the system (queue plus server).
Under many natural service distributions like Poisson or geometric, under all policies, there is a positive probability of reaching state 0, i.e., empty queue, in one step. Therefore, our assumption is satisfied. 
However, since the time to reach some other states (e.g., those corresponding to very large queue sizes) can be very large or even unbounded depending on the policy, the worst-case hitting time may be arbitrarily large. 
Similarly, consider {\it inventory control} or {\it news-vendor} problems with the popular class of base stock policies (e.g., see~\cite{federgruen1984computational,agrawal2022learning}). There, for any base stock policy, under very mild assumptions, some states, like empty stock, are reachable in a small expected time for all admissible policies, making $H$ small. However, the states with total inventory above a policy's base stock level can never be reached under that policy; therefore, the MDP is not ergodic, and the worst-case hitting time can be very large or unbounded. 
Finally, consider applications where some {\it task is attempted repeatedly} with the objective of maximizing the total number of successes in time $T$. Such examples include many {\it game-playing and robotic applications} for which state-of-the-art literature often resorts to the episodic RL setting. In such settings, an episode is typically unlikely to last beyond a certain maximum number (say $H$) of steps, thus satisfying our assumption. However, the episode lengths are rarely deterministic and can vary significantly depending on the policy deployed. Thus, our formulation can capture such settings more faithfully than the episodic RL formulation with fixed-length episodes. 


Among technical contributions, 
we propose a novel $\Lbar$ operator defined as $\Lbar v = \frac{1}{H} \sum_{h=1}^H L^h v$ where $L$ denotes the standard Bellman operator with discount factor $1$, and $L^h$ denotes $h$ applications of this operator. We show that under the given assumption, the $\Lbar$ operator has a strict contraction (in span). 
In contrast, strict contraction of $L$ operator does not hold under discount factor $1$ unless very restrictive structural assumptions are imposed on the transition matrix. Some texts on average reward settings (e.g., \cite{puterman2014markov} Chapter 8.5) assume $J$-stage contraction, i.e.,  strict contraction of $L^J$ operator for some $J$. We give simple examples in Appendix~\ref{app: putterman contraction operator comparison} showing that this is much more restrictive than our setting.
 Our new way of achieving strict contraction in average reward RL may be of independent interest. 

Our algorithm design builds upon the  Q-learning algorithm while replacing the Bellman operator with the novel $\Lbar$ operator. It uses ideas from the optimistic Q-learning algorithm in episodic settings \citep{jin2018q} to estimate $L^1, L^2, \ldots, L^H$ and thereby $\Lbar$. Our model-free algorithm improves the existing literature both in terms of regret bounds and simplicity of algorithmic design.  Specifically, our algorithm achieves a regret bound of $\tilde{O}(H^5 S\sqrt{AT})$ in the average reward setting\footnote{Here, $\tilde{O}$ hides logarithmic factors in $H,S,A, T$, and any additive lower order terms in $T$ have been omitted.}. A regret bound of $\tilde{O}(H^6 S\sqrt{AT})$ in the episodic setting follows as a corollary of this result. 

\paragraph{Organization.} In the next section, we formally present our setting, main results, and comparison to related work. Algorithm design and an overview of challenges and techniques are presented in Section~\ref{sec: algorithm}. Section~\ref{sec: regret analysis} outlines our regret analysis. All missing proofs are in the appendix.



 \section{Our setting and main results}

\subsection{Our setting: average reward MDP with a frequent state}
\label{sec: setting}
We consider a Reinforcement Learning (RL) problem with an underlying Markov Decision Process (MDP) described by the tuple ${\cal M}=(\calS,\calA,P,R)$, where $\calS,\calA$ are the state space and the action space of finite size $S,A$ respectively; $P$ is the transition matrix; and $R$ is the reward function. 
The problem proceeds in discrete and sequential time steps, $t=1,\ldots, T$. At each time step $t$, the RL agent observes a state $s_t\in \cal S$ and takes an action $a_t \in \calA$. The agent then observes a new state $s_{t+1}$ generated according to the transition probability $\Pr(s_{t+1}=s'|s_t,a_t) = P_{s_t,a_t}(s')$, and then receives a bounded reward\footnote{Our results can be extended to unbounded Gaussian or sub-Gaussian rewards using the standard techniques in the literature.} $r_t\in [0,1]$ with $\Ex[r_t|s_t,a_t]=R(s_t,a_t)$. The transition and reward models $P,R$ of the underlying MDP ${\cal M}$ are apriori unknown to the agent. 

We assume that the MDP ${\cal M}$ satisfies \emph{either} of the following key assumptions. 
 \begin{assumption}
\label{assume: expected hitting time assumption}
There exists a state $s_0 \in {\cal S}$ such that under any policy (stationary or non-stationary), starting from any state $s\in {\cal S}$, the expected time to visit the state $s_0$ is upper bounded by $H$.
\end{assumption}
\begin{assumption}\label{assume: hitting time assumption H p}
There exists a state $s_0 \in {\cal S}$ such that under any policy (stationary or non-stationary), starting from any state $s\in {\cal S}$, the probability to visit the state $s_0$ in time $H$ is at least $p$ for some constant $p>0$. 
\end{assumption}

In fact, it is easy to show that the two assumptions are equivalent as in Lemma~\ref{lem: assumption connection}. However, one may be more intuitive to use than the other depending on the application. A proof of this lemma is provided in Appendix \ref{app: preliminaries}.
\begin{restatable}{lem}{lemAssum}
\label{lem: assumption connection}
Assumption \ref{assume: expected hitting time assumption} implies that Assumption \ref{assume: hitting time assumption H p} holds with parameters $(2H, \frac{1}{2})$. And, Assumption \ref{assume: hitting time assumption H p} implies that Assumption \ref{assume: expected hitting time assumption} holds with parameter $H/p$. 
\end{restatable}

Under either of the two assumptions, the MDP ${\cal M}$ is  {\it unichain} (see \cite{puterman2014markov} for definition). We provide a proof in Appendix \ref{apx: weakly communicating} (Lemma \ref{lem: unichain}) for completeness. A unichain MDP setting is more restrictive than a weakly communicating MDP, but can be significantly less restrictive than ergodic MDPs for many applications. For instance, consider the queuing and inventory examples discussed in the Introduction. The underlying MDP is unichain and satisfies our assumption, but is not ergodic since for many natural policy classes certain queuing and inventory states may not be reachable under some policies. 


Given that the underlying MDP satisfies Assumption \ref{assume: expected hitting time assumption} or \ref{assume: hitting time assumption H p}, the goal of the RL agent is to optimize the total reward over a horizon $T$, or equivalently minimize \emph{regret} which compares the algorithm's total reward to the optimal. In average reward settings, regret is defined with respect to the optimal asymptotic average reward (aka gain) $\rho^*$ for MDP ${\cal M}$. 
 For unichain and weakly communicating MDPs, the gain $\rho^*$  is independent of the starting state and achieved by a stationary policy (see Theorem 8.3.2 of \citet{puterman2014markov}). That is, 
$$ \textstyle \rho^* = \max_{\pi:\calS \rightarrow \Delta^\calA} \rho^\pi(s); \text{ where } \rho^\pi(s) \coloneqq \lim_{T\to \infty}\frac{1}{T}\Ex\sbr{\sum_{t=1}^{T}R(s_t,a_t)|s_1=s; a_t\sim\pi(s_t)},$$ 
for all $s$. Then, in line with the existing literature, we define regret as 
$$
\textstyle \Reg(T) \coloneqq \sum_{t=1}^T (\rho^*-R(s_t,a_t)).
$$
Let $\pi^*$ denote the optimal stationary policy. Then, along with the so called `bias vector' (aka value vector) $V^*$, defined as 
$$ \textstyle V^*(s) = \lim_{T\to \infty}\frac{1}{T}\Ex\sbr{\sum_{t=1}^{T}(R(s_t,a_t)-\rho^*)|s_1=s; a_t=\pi^*(s_t)},$$
$\rho^*$ satisfies following Bellman equations (see \cite{puterman2014markov}, Chapter 8) which connects it to dynamic programming based algorithms like Q-learning: 
$$
\rho^*  + V^*(s) = \max_a R(s,a) + P_{s,a}\cdot V^* , \forall s.
$$
Or equivalently, using the Bellman operator $L$ defined as $[Lv](s) := \max_{a} R(s,a) + P_{s,a} \cdot v$, 
the Bellman equations in the average reward setting can be written compactly as:
$$ \rho^* \one + V^* = LV^*.$$
Here $\one$ denotes the vector of all ones. Clearly from above, we have $\spa(V^*-LV^*)=0$, where for any actor $v$, $\spa(v) := \max_s v(s) - \min_s v(s)$. The value vector $V^*$ is known to be unique up to any constant shift.  Furthermore, we show that under Assumption \ref{assume: expected hitting time assumption} and \ref{assume: hitting time assumption H p}, respectively, the span of the vector $V^*$ is bounded as: $\spa(V^*) \le 2H$ and $\spa(V^*)\le 2H/p$.   (See Lemma \ref{lem:spanVstar} in Appendix \ref{app: preliminaries}).  

\paragraph{Episodic setting as a special case.} We remark that the episodic setting forms a special case of the average reward setting presented here. Intuitively, our assumptions are satisfied by any episodic MDP with episode length $H$ since the terminal state/starting state $s_0$ is visited after every set of $H$ steps. A caveat is that in the literature (e.g., see \citet{jin2018q,domingues2021episodic}) episodic settings typically have time-inhomogeneous MDP, i.e., there is a different reward and transition model $R^h, P^h$ at every index $h=1,\ldots, H$ of an episode. Therefore, reduction to a time-homogeneous average reward MDP requires increasing the state space size from $S$ to $SH$. In Appendix \ref{app: epsiodic setting description}, we provide a formal reduction showing that the regret minimization problem in the episodic settings with fixed episode length $H$, state space of size $S$, and $T/H$ episodes, can be reduced to our average reward setting with time horizon $T$ and state space of size $SH$. Thus, our algorithm and regret bounds will apply directly to the episodic setting with $S$ replaced by $SH$.

\if 1
\subsection{Episodic setting as a special case}
\label{sec: episodic case}
We show that the problem of regret minimization in episodic MDPs forms a special case of our setting. Consistent with the literature (e.g., \citep{jin2018q,domingues2021episodic}), we define the episodic setting using a time-inhomogeneous MDP $M = (\calS,\calA, P, R, H)$, where  $(P, R)= \{P^h,R^h\}_{h=1}^H$. 
 Under any policy, after exactly $H$ steps, the MDP reaches the terminal state (say $s_0$) which is an absorbing state with reward $0$. 
The value function $V^{\pi}_h(s)$ of (possibly non-stationary) policy $\pi=(\pi_1,\ldots, \pi_H)$ at step $h$ is defined as the $H-h+1$ step expected reward starting from state $s$.
Then optimal value for any $h=1,\ldots, H$ is given by, 
$$\textstyle
V^*_h(s) = \max_\pi V^\pi_h(s); \text{ where } V^\pi_h = \Ex[\sum_{j=h}^H R^h(s_j,\pi_h(s_j)) | s_h=s].
$$
Unlike unichain or weakly communicating MDPs, here the optimal value depends on the starting state and the optimal policy is non-stationary. The long-term regret minimization problem in episodic setting seeks to minimize total regret over $T/H$ episodes i.e., 
$$
\textstyle
\Reg^{\mbox{\tiny{episodic}}}(T) =  \frac{T}{H} V^*_1(s_1) - \sum_{k=1}^{T/H} \sum_{h=1}^H R^h(s_{k,h},a_{k,h}),
$$
where $s_{k,h}, a_{k,h}$ denote the state, action at step $h$ in $k^{th}$ episode. 
It is easy to reduce the above problem to the average reward setting with MDP  ${M}'$ that has a slightly larger ($HS$) state space. Simply construct MDP ${M}'$ by augmenting each state $s$ with indices $h=1,\ldots, H$,
and modifying the transition model so that $s_0$ is not an absorbing state but instead transitions to the starting state $s_1$ with probability $1$. Then, we show (Appendix \ref{app: epsiodic setting description}) that $\Reg_{M}^{\mbox{\tiny{episodic}}}(T) = \Reg_{M'}(T)$ and $M'$. Importantly, $M'$ satisfies Assumption~\ref{assume: expected hitting time assumption} and \ref{assume: hitting time assumption H p}, 
so that our algorithm and regret analysis directly applies to the episodic setting with $S$ replaced by $HS$. 
In Appendix \ref{app: epsiodic setting description}, we also illustrate the connection between the Bellman equations for the two settings which may be of pedagogical interest. 
\fi
 
\subsection{Main results}
\label{sec: main results}
Our main contribution is a model-free optimistic Q-learning algorithm for the average reward setting described in Section~\ref{sec: setting}. We prove the following regret bound for our algorithm.

\if 1
A main difficulty in designing algorithms for average reward setting is that the $L$ operator does not have a strict contraction property when discount factor is $1$. Our key novel insight is that under  Assumption~\ref{assume: expected hitting time assumption}\red{ add Assumption \ref{assume: hitting time assumption H p}}, an operator $\Lbar$ that we define as 
$$\textstyle \Lbar v := \frac{1}{H} (L^H v + L^{H-1} v + \cdots + L v),$$
has a strict contraction property in span. Specifically, we prove (see Lemma \ref{lem: span contraction property})\footnote{Obtained by substituting $(H,p)$ by $(2H, \frac{1}{2})$ in Lemma \ref{lem: span contraction property}. For completeness, in Lemma~\ref{lem: span contraction property general}, we also prove a more standard form of span contraction property: for all $v_1,v_2, \ \spa(\Lbar v_1-\Lbar v_2) \le \del{1-\frac{1}{4H}}\spa(v_1 - v_2)$.} 
that for any $v\in \mathbb{R}^S$, 
\begin{equation}
\label{eq:spanContration1}
   \textstyle \spa\del{\Lbar v - V^*} \leq \del{1-\frac{1}{4H}}\spa\del{ v - V^*}.
\end{equation}
This result forms the basis of our algorithm design that uses ideas from episodic Q-learning to estimate $\Lbar$ operator as an average of $L^h$ operators for $h=1,\ldots, H$. 
Our algorithm assumes the knowledge of $T, S, A, H$ but does not need to know the identity of the frequent state $s_0$. Our main result is the following regret bound for our algorithm. 
\fi
\begin{restatable}[Average reward setting]{theorem}{thmMain1}
\label{thm: main regret general}
Given Assumption \ref{assume: expected hitting time assumption} or \ref{assume: hitting time assumption H p}, there exists an optimistic Q-learning algorithm (specifically, Algorithm \ref{alg:main})
that achieves a regret bound of
\begin{align*}
\Reg(T) & = \regBoundOtilde 
\end{align*}
for any starting state $s_1$, with probability $1-\delta$.
\end{restatable}
Note that when $T\ge H^{8}S^2A$, we get an $\tilde O(H^5S\sqrt{AT})$ regret bound. As discussed earlier, the episodic setting with $T/H$ episodes can be reduced to our setting, with an increase in the size of state-space from $S$ to $SH$ in case of time-inhomogeneous episodic MDPs. Therefore, we obtain the following corollary for the episodic setting.
\begin{restatable}[Episodic setting]{corollary}{thmEpisodic}
\label{thm: main regret episodic}
For episodic settings with fixed episode length $H$ and $T/H$ episodes, our algorithm achieves a regret bound of $\tilde{O}(H^6S\sqrt{AT} + H^{9}S^2A)$.
\end{restatable}
Admittedly, the above regret bound is not optimal for the episodic setting where optimistic $Q$-learning algorithms have achieved near-optimal regret bounds of $\tilde{O}(H\sqrt{SAT})$ \citep{li2021breaking}. 
Those algorithms, however, fundamentally rely on the fixed length episodic structure and cannot be applied to the average reward setting even when assumptions~\ref{assume: expected hitting time assumption} and~\ref{assume: hitting time assumption H p} hold. On the other hand, most existing algorithms for average reward settings make assumptions like bounded diameter and the worst-case hitting time (for all states) that are not satisfied by the episodic settings (see Section \ref{sec: related work} for further details). Our work provides a unique unified view of the two settings, along with an algorithm that applies to both paradigms. Furthermore, as discussed in the next section, it significantly improves upon the state-of-the-art model-free algorithms for average reward settings, both in terms of regret bounds and simplicity of algorithm design.

Finally, our regret bounds also imply a PAC guarantee. Specifically, let $\pi_1,\ldots, \pi_T$ denote the policy used by our algorithm (Algorithm~\ref{alg:main}) in each of the $T$ time steps. Then, we can show that picking a policy $\pi$ randomly from this set (and repeating this experiment multiple times for a high probability guarantee) provides a policy that is $\epsilon$-optimal (i.e., $\rho^*-\rho^\pi \le \epsilon$) with probability $1-\delta$, 
where $\epsilon = \frac{3\Reg(T)}{T} + O(H^2 \sqrt{\frac{S\log(T)\log(1/\delta)}{T}})$. The proof is deferred to Appendix~\ref{apx: PAC}. 
On substituting the regret bound from Theorem~\ref{thm: main regret general}, we get an $(\epsilon,\delta)$-PAC policy using $\tilde O(\frac{H^{10}S^2A}{\epsilon^2})$ samples.
\subsection{Comparison to related work} 
\label{sec: related work}
Our work falls under the umbrella of \textit{online reinforcement learning},
specifically on regret minimization in tabular average reward settings with a weakly communicating MDP. \citet{jaksch2010near} proved a regret lower bound of $\Omega(\sqrt{DSAT})$ for this setting where $D$, referred to as the diameter of the MDP, bounds the time to reach any recurrent state from another state under some policy. 

Most of the earlier works on this topic focus on \emph{model-based algorithms} with near-optimal regret guarantee of $\tilde O(DS\sqrt{AT})$ \citep{jaksch2010near, agrawalJia2017}, or $\tilde O(H^*S\sqrt{AT})$ \citep{bartlett2012regal}.
Recently, several papers (e.g., \citet{fruit2018efficient}, \cite{zhang2019bias}) improve the dependence on $S$ and $D$, with  \citet{zhang2019bias} closing the gap to achieve an $\tilde O(\sqrt{H^*SAT})$ regret bound where $H^* = \spn(V^*) \le D$; however their algorithm is not efficiently implementable. Very recently, \citet{boone2024achieving} claimed to have an algorithm that is tractable and achieves the minimax optimal regret of $\tilde O(\sqrt{H^*SAT})$.

There has been an increased interest in designing \emph{model-free algorithms} with provable regret bounds for the average reward setting. 
However, unlike the episodic MDP, where variants of Q-learning have been shown to achieve near-optimal regret bounds \citep{jin2018q, li2021breaking}, there is still a significant gap between model-free and model-based algorithms in average reward settings. Table \ref{table: model-free} lists state-of-the-art regret bounds for model-free algorithms (when applied to the tabular MDP case). The table may not be comprehensive, but it highlights the most relevant recent results.

\citet{wei2020model} presented a simple extension of episodic optimistic Q-learning from \cite{jin2018q} to the average reward case with a regret bound that grows as $T^{2/3}$. Most subsequent works made more restrictive assumptions in order to achieve a $\sqrt{T}$ regret bound. These include a bound on the mixing time for all policies ($t_{\text{mix}}$), and a bound on the time to reach any state from any other state  under any stationary policy  ($t_{\text{hit}}$).  

\begin{table*}[h!]
{
\begin{center}
    \begin{tabular}{|c|c|}
    \multicolumn{2}{c}{Average reward}\\
    \hline
  Optimistic Q-learning {\small\citep{wei2020model}} & $\otilde(H^*(SA)^{\frac{1}{3}}T^{\frac{2}{3}})$  \\
  MDP-OOMD {\small\citep{wei2020model}}& $\otilde(\sqrt{t_{\text{mix}}^3 \eta AT})$   \\
     MDP-EXP2 {\small\citep{wei2021learning}}& $\otilde(\frac{1}{\sigma}\sqrt{t_{\text{mix}}^3 T})$ \\
    UCB-AVG {(\small \citep{zhang2023sharper})} & $\tilde{O}(S^5A^2H^*\sqrt{T})$ \\
    \textbf{\algoname{}  [this paper]}& $\otilde(H^5S\sqrt{AT})$\\
   \hline
   \multicolumn{1}{c}{}\\
    \multicolumn{2}{c}{Episodic}\\
   \hline
     Optimistic Q learning {\small\citep{jin2018q}}&  $\otilde(\sqrt{H^3SAT})$  \\
    Q-EarlySettled-Advantage {\small\citep{li2021breaking}}& $\otilde(\sqrt{H^2SAT})$ \\
    \textbf{\algoname{}  [this paper]}& $\otilde(H^6S\sqrt{AT})$\\
   \hline
    \end{tabular}
\end{center}
}
\caption{Comparison of our results to literature on model-free algorithms. Note that by preceding discussion, $\eta \ge S, \frac{1}{\sigma} \ge S$. Also, $H \le t_{\text{mix}} + 2\eta$ and $H \le t_{\text{mix}} + \frac{2}{\sigma}$. }
\label{table: model-free}
\end{table*}
Specifically, \citet{wei2020model} assumes an upper bound of $\eta$ on the quantity $ \max_\pi \sum_s \frac{\mu^*(s)}{\mu^\pi(s)}$ where $\mu^\pi,\mu^*$ are the stationary distributions of policy $\pi$ and optimal policy respectively, so that $\eta\ge S$. 
The regret bounds derived in several works in the linear function approximation setting (e.g., \citet{hao2021adaptive, wei2021learning, abbasi2019politex, abbasi2019exploration})  involve a parameter $\sigma$ when applied to the tabular case. This parameter $\sigma$ lower bounds the probability to visit any state and action under any stationary policy so that $\frac{1}{\sigma} \ge SA$. 
In comparison to the above literature, our work only assumes a bound $H$ on hitting \emph{one} frequent state ($s_0$) and does not require uniform mixing. Given $t_{\text{mix}}=t_{\text{mix}}(\epsilon)$ for $\epsilon\le\frac{1}{2t_{\text{hit}}}$, it is easy to show that our Assumption \ref{assume: expected hitting time assumption} is strictly weaker and holds in these settings with $H=t_{\text{mix}}+2t_{\text{hit}}$, and similarly for $t_{\text{hit}}$ replaced by $\eta, \frac{1}{\sigma}$. In practice, $H$ can be much smaller than $t_{\text{hit}}, \eta, \frac{1}{\sigma}$ especially when the state space is large and not all policies explore all states uniformly. Indeed, in the queuing and inventory control examples that we discussed in the introduction, $H$ is small or a constant but $t_{\text{hit}}$ can be very large. 

A notable exception to the above literature is the work by \citet{zhang2023sharper}
that achieves a $\sqrt{T}$ regret (specifically $\tilde O(S^5A^2H^*\sqrt{T})$) while making only the weakly communicating assumption with a bound $H^*$ on $\spa(V^*)$, which is the weakest possible assumption for a sub-linear regret in the average reward settings. However, besides having a high dependence on the size of the state space $S$ (specifically $S^5$) in their regret bound, some design features like tracking pair-wise state-visit counters arguably make their algorithm closer to a model-based algorithm. Their cleverly designed algorithm ensures that the number of quantities tracked is still $O(SA)$, so that the algorithm technically qualifies as model-free.
In contrast, our algorithm keeps the basic structure of an optimistic Q-learning algorithm intact and achieves a regret bound that has a linear dependence on $S$. 

\if 1
\begin{table*}[h]
{
    \begin{tabular}{|c|c |c|c|}
    \hline
 & \textbf{Algorithm} & \textbf{Regret} & \textbf{Comment} \\ \hline
       \multirow{8}{*}{Average-reward} 
 & Politex {\small\citep{abbasi2019politex}} & $\otilde(\frac{t_{\text{mix}}}{\sigma^3}\sqrt{SA}T^\frac{3}{4})$ & \makecell[c]{all states recurrent} \\
 & Optimistic Q-learning {\small\citep{wei2020model}} & $\otilde(H^*(SA)^{\frac{1}{3}}T^{\frac{2}{3}})$ &- \\
 & MDP-OOMD {\small\citep{wei2020model}}& $\otilde(\sqrt{t_{\text{mix}}^3 \eta AT})$ & all states recurrent \\
  & AAPI {\small\citep{hao2021adaptive}}& $\otilde(\frac{t_{\text{mix}}^2}{\sqrt{\sigma}}(\eta SA)^{\frac{1}{3}}T^{\frac{2}{3}})$ & all states recurrent \\
    & MDP-EXP2 {\small\citep{wei2021learning}}& $\otilde(\frac{1}{\sigma}\sqrt{t_{\text{mix}}^3 T})$ & all states recurrent \\
   & UCB-AVG {(\small \citep{zhang2023sharper})} & $\tilde{O}(S^5A^2H^*\sqrt{T})$& high S dependence \\
   &&&\\
   \multirow{2}{*}{Episodic} &  Optimistic Q learning {\small\citep{jin2018q}}& $\otilde(\sqrt{H^3SAT})$ & - \\
   & Q-EarlySettled-Advantage {\small\citep{li2021breaking}}& $\otilde(\sqrt{H^2SAT})$ & - \\
   \hline
   \hline
   \multirow{1}{*}{Our work} &  \textbf{\algoname{}  this work}& $\otilde(H^5S\sqrt{AT})$ & one state recurrent \\
   \hline
    \end{tabular}
\caption{Model-free results. $\eta$ in~\cite{wei2020model,hao2021adaptive} denotes distributional mismatch coefficient is the worst case of ratio of the probabilities induced by stationary distributions induced by the optimal policy to any other stationary policy, summed over all states and $\eta \geq S$; $\sigma$ in results with linear function approximation, e.g.,~\cite{abbasi2019politex,hao2021adaptive,wei2021learning}, ensures every stationary policy visits every state-action with a non-zero probability, therefore $\frac{1}{\sigma}\geq t_{\text{hit}}$}
}
\label{table: model-free}
\end{table*}
\fi

\section{Algorithm design} 
\label{sec: algorithm}
Our algorithm extends the Optimistic Q-learning approach of \cite{jin2018q} to a more general setting that includes both the episodic and the non-episodic settings that satisfy Assumption~\ref{assume: expected hitting time assumption} or \ref{assume: hitting time assumption H p}. 
Following Lemma \ref{lem: assumption connection} that shows that the two assumptions are essentially equivalent, in the remaining paper we work with Assumption \ref{assume: hitting time assumption H p} only. The corresponding results under Assumption \ref{assume: expected hitting time assumption} can be derived by simply substituting $(H,p)$ by $(2H,1/2)$.

\subsection{Algorithm design overview}
At its core, our algorithm is trying to estimate the optimal bias/value vector $V^*$ which in average reward setting is any vector that satisfies the  Bellman equation, written compactly as (see section \ref{sec: setting}):
$$\spa(V^*-LV^*)=0.$$
The algorithm maintains an estimate $\Vbar$ for this optimal value vector, which is improved by repeatedly applying a value-iteration type update of the form
$$\Vbar \leftarrow \Lbar^{\ } \Vbar.$$
Here, operator $\Lbar$ is as defined before: for any vector $v \in \R^S$,
$$\Lbar v := \frac{1}{H}(Lv+L^2v+\cdots L^H v).$$ Our main new technical insight behind this approach is that under Assumption \ref{assume: hitting time assumption H p}, the $\Lbar$ operator has the following strict span contraction property. 
A proof is in Appendix~\ref{apx: span contraction}. 
\begin{restatable}[Span Contraction]{lem}{spanContractionCorollary}
\label{lem: span contraction property}
Under Assumption~\ref{assume: hitting time assumption H p}, for any 
vector $v\in \mathbb{R}^S$, 
\begin{equation*}
   \spa(\overline L v- V^*) \leq (1-\frac{p}{H})\ \spa( v - V^*).
\end{equation*}
\end{restatable}

Therefore, by repeatedly applying $\Vbar \leftarrow \Lbar^{} \Vbar$, the algorithm can ensure that $\Vbar$ gets closer and closer to $V^*$ in span. Specifically, after $k$ such updates, we would have $\spa(\Lbar^k \Vbar - V^*)\le (1-\frac{p}{H})^k \spa(\Vbar - V^*) \rightarrow_{k\rightarrow \infty} 0,$ with linear convergence rate.

However, to apply this $\Lbar$ operator without knowing the model apriori, the algorithm needs to estimate $L^h \Vbar$ for $h=1,\ldots, H$, using observations. This is done using $Q$-learning. Specifically, the algorithm maintains quantities $Q^{h}(s,a)$ and $V^{h}(s)$ for $h=1,\ldots, H$ and all $s,a$. At any time step $t$, the algorithm observes the current state,  action, reward, and the next state, i.e., the tuple $(s_t,a_t, r_t, s_{t+1})$, and performs the following (optimistic) $Q$-learning type updates for \mbox{$h=H, H-1,\ldots, 1$}:
\begin{itemize}
    \item $Q^h(s_t,a_t) \leftarrow (1-\alpha_t) Q^h(s_t,a_t) + \alpha_t (r_t+ V^{h+1}(s_{t+1})) + \text{exploration-bonus}$
    \item $V^h(s_t) \leftarrow \max_a Q^h(s_t,a)$
\end{itemize}
where $\alpha_t$ is a carefully chosen learning rate and $V^{H+1}$ is initialized as $\Vbar$. Given enough samples and sufficient exploration of states and actions, such $Q$-learning updates ensure that $V^h$ is close to $LV^{h+1}$ so that 
$$V^h \approx LV^{h+1} \approx L^{H-h+1} V^{H+1} = L^{H-h+1} \Vbar.$$ 
Then, the algorithm updates $\Vbar$ roughly as the average of these updated $V^{h}$, 
so that
$$ \textstyle \Vbar \approx \frac{1}{H} \sum_{h=1}^H V^h \approx  \frac{1}{H} \sum_{h=1}^H  L^{H-h+1} \Vbar = \Lbar^{} \Vbar,$$
thus approximately implementing the update $\Vbar \leftarrow \Lbar^{} \Vbar$.
To achieve sufficient exploration, the algorithm adds a UCB-type exploration bonus to the Q-learning update step. Furthermore, given a state $s_t$, it chooses the action $a_t$ as arg max of $Q^{t,h_t}(s_t,a)$ where $h_t$ is picked uniformly at random from $\{1,\ldots, H\}$. This random choice of $h$ is important in order to ensure that in every state, the arg max actions for all $h$ are explored enough. 

In the next section, we provide algorithmic details that include the precise descriptions of the exploration bonus and learning rates, which are similar to \citet{jin2018q}. We also add a careful epoch-based design, which is typical for learning in the average reward settings.

\begin{algorithm}[h!]
\begin{algorithmic}[1]
\SetAlgoLined
\STATE \textbf{Input:} Parameters $(H,p, H^*, \delta)$. 
\STATE \textbf{Initialize: } $\Vbar(s)=H, \forall s$,
$\tau_\ell=0, \ell\geq 0$, $N_0(s)=0,  \forall s \in \calS$. $t=1$.
\vspace{0.05in}
\FOR{  $\ell=1,2,\ldots $}
\vspace{0.02in}
\STATE \textbf{Reset:}  For all $s,a$, set $V^{H+1}(s)=\Vbar(s)$,  
  $Q^{h}(s,a)=V^h(s)= \max_{s'} \Vbar(s')+ H, h =1,\ldots, H$,  , $\Vbar(s) \leftarrow \max_{s'}\Vbar(s')+H$. Also,  
reset $N(s)=0, N(s,a)=0$. 
\vspace{0.08in}
\REPEAT
\vspace{0.08in}
\STATE \algoComments{ /* Play arg max action of $Q^h$ for a randomly selected $h$ */\\}
\STATE Observe $s_t$. 
\STATE Generate $h_t\sim \texttt{Uniform} \{1,\ldots, H\}$. Play $a_t=\arg \max_a Q^{h_t}(s_t,a)$.
\STATE Observe reward $r_t$ and next state $s_{t+1}$.
\STATE $n:=N(s_t,a_t)\leftarrow N(s_t,a_t)+1$; $N(s_t)\leftarrow N(s_t)+1$.

\vspace{0.08in}
\STATE \algoComments{ /*Use the observed reward and next state to update $Q^h$-values for all $h$ */\\}
 
\FOR{ $h=H,H-1,\ldots,1$}
\STATE $Q^{h}(s_t,a_t) \leftarrow (1-\alpha_n) Q^{h}(s_t,a_t) + \alpha_n (r_t + V^{h+1}(s_{t+1})+b_n). $
\STATE $V^{h}(s_t) \leftarrow \max_a Q^{h}(s_t,a).$
\ENDFOR
\vspace{0.08in}
\STATE \algoComments{ /* Update $\Vbar$ to track the (running) average of $\frac{1}{H} \sum_{h=1}^H V^{h}$ over an epoch*/
\\}
\vspace{0.04in}
\STATE $\textstyle \overline{V}(s_t) \leftarrow \frac{1}{N(s_t)} \cdot \overline{V}(s_t) + \del{1-\frac{1}{N(s_t)}}  \del{\frac{1}{H} \sum_{h=1}^H V^{h}(s_t)}.$
\vspace{0.08in}
\STATE  $t\leftarrow t+1$; $\tau_\ell\leftarrow \tau_\ell+1$.\\
\vspace{0.08in}

\UNTIL{ $N(s) \ge (1+1/C )N_{\ell-1}(s)$ for some $s$ or $ \tau_\ell \ge (1+1/C ) \tau_{\ell-1}$}
\vspace{0.08in}
\STATE \algoComments{/* Occasionally, clip $\Vbar$ to control its span */}
\STATE If $(\ell \text{ mod } K) = 0$,  set $\overline{V}\leftarrow  \spaProj^{} \overline{V}$. 
\vspace{0.08in}
\STATE For all $s$, set $N_{\ell}(s) \leftarrow N(s)$. \algoComments{/*Book-keeping */} \\
\ENDFOR
\caption{Optimistic Q-learning for average reward MDP}\label{alg:main}
\end{algorithmic}
\end{algorithm}

\subsection{Algorithm details}
Algorithm \ref{alg:main} provides the detailed steps of our algorithm. It takes as input, parameters $(H,p)$ satisfying Assumption \ref{assume: hitting time assumption H p}, and an upper bound $H^*$ on $\spa(V^*)$. (Note that by Lemma \ref{lem: boundedBiasAllPolicies}, $\spa(V^*)\le 2H/p$ which can be used if $H^*$ is not known). We assume that the time horizon $T$ and the size of state space and action space $S,A$ are also fixed and known. The algorithm uses these along with the parameters $(H,p,H^*, \delta)$ to define quantities $C,K, b_n, \alpha_n$ as follows: 
$$\textstyle K:=\constK, C:=\constC,$$
and for any $n\ge 1$,
$$\textstyle b_n := \alphaBonus{n}, \alpha_n :=\frac{C+1}{C+n}.$$

The algorithm proceeds in epochs $\ell=1,2,\ldots$ of geometrically increasing duration.  The epoch break condition (see \texttt{line 19}) is such that the total number of epochs is upper bounded by $\zeta:= CS\log(T) = O(\frac{1}{p} H^2 S \log^2(T))$.

 In each epoch, the algorithm resets and re-estimates Q-values and V-values $Q^h(s,a), V^h(s)$ for all $s,a$ and $h=1,\ldots, H$, starting from some initial values that are well above the current $\Vbar(s)$. Specifically, in the beginning of every epoch $\ell$ (see \texttt{line 4}), $Q^h(s,a), V^h(s)$ for all $s,a,h=1,\ldots, H$ are initialized as $\max_{s'} \Vbar(s')+ H$.
Then, in the beginning of every round $t$ of epoch $\ell$ (\texttt{line 7-9}), the algorithm observes the state $s_t$ and uniformly samples $h_t\in \{1,\ldots, H\}$. Action $a_t$ is picked as the arg max action: $a_t = \arg\max_a Q^{h_t}(s_t,a)$. Picking $h$ uniformly at random is important here in order to ensure sufficient exploration of the arg max actions of $Q^h(s_t,\cdot)$ for all $h$. 
On playing the action $a_t$, the reward $r_t$ and the next state $s_{t+1}$ is observed.

The tuple $(s_t,a_t,s_{t+1}, r_t)$ is then used to update $Q^h(s_t,a_t), V^h(s_t)$ for all $h\in \{1,\ldots, H\}$ (and not just for $h_t$). For each $h$, a Q-learning style update is performed (\texttt{line 12-15}). 
A subtle but important point to note here is that $Q^h, V^h$ are updated in the reverse order of $h$, i.e., $h=H,H-1,\ldots, 1$ (see \texttt{line 12}). This ensures that the latest updated vector $V^{h+1}$ is used to construct the target for $h$. 
As discussed in the algorithm design overview, the goal of these Q-learning updates is that at the end of an epoch, $V^h$ for each $h$ is a reasonable estimate of $L^{H-h+1} \Vbar$.

The updated $V^h(s_t),\,h=1,\ldots, H$ are then used in \texttt{line 17} to update $\Vbar(s_t)$ so that it tracks the (running) average of $\frac{1}{H}\sum_{h=1}^H V^h(s_t)$ over the current epoch.  
Combined with the observation made earlier that $V^h$ estimates $L^{H-h+1} \Vbar$ for $h=1,\ldots, H$, this means that at the end of an epoch, the algorithm has roughly executed a value iteration like update $\Vbar \leftarrow \frac{1}{H} \sum_h L^{H-h+1} \Vbar = \Lbar\ \Vbar$. 

Our regret bound will utilize that over $\zeta:= CS\log(T)$ epochs enough of these value iteration-like updates are performed so that $\Vbar$ will be close (in span) to the optimal value vector $V^*$. 


A final detail is that occasionally (every $K=O(\frac{H}{p}\log(T))$ epochs), a projection operation  $\Vbar \leftarrow \spaProj \ \Vbar$ is performed on $\Vbar$ to control its span (see \texttt{line 21}). This projection operator defined for any $v\in \mathbb{R}^S$ as: 
\vspace{-0.08in}
\begin{equation}
\label{def: new projection operator}
\textstyle [\spaProj v](s) := \min\{2H^*, v(s)-\min_{s\in \calS} v(s)\} + \min_{s\in \calS} v(s),
\end{equation}
\vspace{-0.1in}
trims down the {\it span} of vector $\Vbar$ to at most $2H^*$. 

\section{Regret Analysis}\label{sec: regret analysis}
In this section we analyze the regret of Algorithm \ref{alg:main} under Assumption \ref{assume: hitting time assumption H p}. Given the connection between the two assumptions observed in Lemma \ref{lem: assumption connection}, all the results proven here are applicable under Assumption \ref{assume: expected hitting time assumption} by replacing $(H,p)$ by $(2H,1/2)$.
Specifically, we prove the following theorem which directly implies the main result stated in Theorem \ref{thm: main regret general}.

\begin{restatable}{theorem}{thmMain}
\label{thm: main regret}
Under Assumption \ref{assume: hitting time assumption H p},  with probability at least $1-\delta$, the $T$ round regret of Algorithm~\ref{alg:main} (with parameters $(H,p, \frac{2H}{p}, \frac{\delta}{3})$) is upper bounded as:
\begin{eqnarray*}
    \Reg(T) & = & \textstyle \regBoundO \\
    & = & \regBoundOtilde.
\end{eqnarray*}
\end{restatable}
In the rest of this section, we provide a sketch of the proof of this theorem. All the missing proofs from this section are in Appendix \ref{apx: regret analysis}.

Let $\Vbar^{\ell}$ denote the value of $\Vbar$ at the beginning of epoch $\ell$, and let $V^{t,h}, Q^{t,h}$ denote the value of $V^h, Q^{h}$ at the beginning of round $t$  (i.e., {\it before} the updates of round $t$) in Algorithm \ref{alg:main}. 

Our first main technical lemma shows that in every epoch, the algorithm essentially performs value iteration like updates $\Vbar^\ell \leftarrow \Lbar^{\ } \Vbar^{\ell-1}$ and $V^{t,h} \leftarrow L^{H-h+1}\Vbar^\ell$  in an optimistic manner. Since the $L$-operator and $\Lbar$-operator are not monotone, therefore $\Vbar^\ell \ge \Lbar^{} \Vbar^{\ell-1}, \Vbar^{\ell-1}\ge \Lbar^{} \Vbar^{\ell-2}$ do not directly imply $\Vbar^\ell \ge \Lbar^2\Vbar^{\ell-2}$, and so on. We state the optimism lemma for $k$ applications of these operators, for certain values of $k$ that we will need for the regret analysis. Here, and throughout this paper, operators $\ge, \le $ denote element-wise comparison when used to compare two vectors.
\begin{restatable}[Optimism]{lem}{optimismNew}
\label{lem:optimismNew}
    With probability $1-\delta$, for all epochs $\ell$ and $t\in \ep \ell$, we have
\begin{eqnarray}
  \Vbar^\ell & \ge & \Lbar^k \Vbar^{\ell-k}, \label{eq:induction1} \\
  V^{t,h} & \ge & L^{H-h+1}  \Lbar^k \Vbar^{\ell-k}, \ h=1,\ldots, H,  \label{eq:induction2}
\end{eqnarray}
for any $k\ge 0$ such that $\ell -k = Kj+1$ for some integer $0\le j\le \jvalue$. 
\end{restatable}
We also obtain an upper bound on the estimated value vectors and Q-value vectors demonstrating that the dynamic programming equation is satisfied approximately (and in aggregate). We  use the following notation: for any time step $t$ and $h=1,\ldots, H$, let $a_{t,h} := \arg \max_a Q^{t,h}(s_t,a)$,  $n_{t,h} = N_t(s_t,a_{t,h})$, and let $\ell_t$ be the epoch that $t$ belongs to.
\begin{restatable}[Upper bound]{lem}{QupperBound}
\label{lem:Q-upper-bound}
With probability at least $1-\delta$,
$$ \sum_{t.h} V^{t,h}(s_t) = \sum_{t,h} Q^{t,h}(s_t, a_{t,h}) \le \sum_{t,h} \left( R(s_t,a_{t,h}) + P_{s_t,a_{t}} \cdot L^{H-h}\overline{L}^{k_t} \Vbar^{\ell_t-k_t} \right)+ C_1   \sum_{t,h} b_{n_{t,h}} + C_2,$$ 
where for all $t$ in an epoch, $\ell_t-k_t=Kj+1$ for some integer $\jvalue\le j \le \lfloor \frac{\ell-1}{K} \rfloor $. 
Here, $C_1=\Cone, C_2=\Ctwo$. 
\end{restatable}

To prove the desired regret bound in Theorem \ref{thm: main regret}, we combine the above upper bound on $\sum_{t,h} V^{t,h}(s_t)$ with the lower bound on each $V^{t,h}(s_t)$ given by the optimism lemma (Lemma \ref{lem:optimismNew}). Combining these upper and lower bounds gives the following inequality,
\begin{equation} 
 \sum_{t,h} L^{H-h+1} \Lbar^{k_t} \Vbar^{\ell_t-k_t} (s_t) \le \sum_{t,h} \left( R(s_t,a_{t,h}) + P_{s_t,a_{t}} \cdot L^{H-h}\overline{L}^{k_t} \Vbar^{\ell_t-k_t} \right)+ C_1   \sum_{t,h} b_{n_{t,h}} + C_2.
\end{equation} 
For notational convenience, denote $v^{t,h}:=L^{H-h}\Lbar^{k_t} \Vbar^{\ell_t-k_t}$.
Then, by moving the terms around in the above we get, 
\begin{eqnarray}
\label{eq:rewardLowerboundMain}
 \sum_{t,h} R(s_t,a_{t,h})  & \ge & \sum_{t,h} \left( Lv^{t,h}(s_t) -P_{s_t,a_t} v^{t,h} \right)- C_1   \sum_{t,h} b_{n_{t,h}} - C_2.
\end{eqnarray} 
The optimal asymptotic average reward $\rho^*$ for the average reward MDP satisfies the following Bellman equation for any state $s_t$ (see Section~\ref{sec: setting}): 
\begin{eqnarray}
\label{eq:BellmanMain}
 \rho^* & = & LV^*(s_t) - V^*(s_t) \nonumber\\
       & = & (LV^*(s_t) - P_{s_t,a_t} V^*) + (P_{s_t,a_t} V^*-V^*(s_t)).
\end{eqnarray}
 Summing \eqref{eq:BellmanMain} over $t,h$, and then subtracting \eqref{eq:rewardLowerboundMain}, we get the following inequality,
\begin{eqnarray}
\label{eq:expectedRegretEquationMain}
HT\rho^* -  \sum_{t,h} R(s_t,a_{t,h})  & \le & \sum_{t,h} \left(LV^*(s_t) - Lv^{t,h}(s_t) - P_{s_t,a_t}(V^*-v^{t,h})\right)\nonumber\\
&&+ \sum_{t,h} (P_{s_t,a_t}V^*-V^*(s_t)) +  C_1   \sum_{t,h} b_{n_{t,h}} + C_2.
\end{eqnarray}
To obtain the regret bound in Theorem \ref{thm: main regret}, we bound each of the terms on the right-hand side of the above inequality. To bound the first term in the above, we use the definition of $L$-operator and observe that for every $t,h$,
$$LV^*(s_t) - Lv^{t,h}(s_t)\le P_{s_t,a}(V^*-v^{t,h}) \text{  for some $a$, so that,}$$
$$LV^*(s_t) - Lv^{t,h}(s_t) - P_{s_t,a_t}(V^*-v^{t,h}) \le (P_{s_t,a}-P_{s_t,a_t})(V^*-v^{t,h}) \le \spa(V^*-v^{t,h}).$$
Then, by the span contraction property of operator $\Lbar$ derived in Lemma \ref{lem: span contraction property}, 
$$
    \spa(V^*-v^{t,h}) \le \spa(V^*-\Lbar^{k_t} \Vbar^{\ell_t-k_t}) \le   (1-\frac{p}{H})^{k_t} \spa(V^*-\Vbar^{\ell_t-k_t}).
$$
where in the first inequality we substituted $v^{t,h}=L^{H-h}\Lbar^{k_t} \Vbar^{\ell_t-k_t}$ and used $\spa(V^*-LV^*)=0$. By the conditions on $k_t$ in Lemma \ref{lem:optimismNew} and Lemma \ref{lem:Q-upper-bound}, for all $t\in \ep \ell$ we have $\ell_t-k_t=Kj+1$ with $j=\jvalue$, since this value of $j$ satisfies the conditions in both lemmas. Due to the algorithm's projection step that is taken at the end of any epoch index of the form $Kj$ and trims down the span of $\Vbar^{Kj+1}$ to $2H^*$, we have that $\spa(V^*-\Vbar^{\ell_t-k_t}) \le \spa(V^*)+\spa(\Vbar^{\ell_t-k_t})\le 3H^*$  for all $t$. To bound the term $(1-\frac{p}{H})^{k_t}$, we observe that for the given value of $j$, $k_t$ is large enough for large $\ell_t$. In particular, $k_t\ge K$ for epochs $\ell_t\ge 2K+1$ so that the term in consideration is smaller than $1/T$ for such epochs and can essentially be ignored; we only need to bound its sum for epochs $\ell_t\le K$ where $k_t=\ell-1$. To bound this sum, we use the epoch-breaking condition and the implied bound on the epoch lengths. Using these observations, we bound 
the first term in \eqref{eq:expectedRegretEquationMain} by $\tilde O(H^3)$. 

The second term $\sum_t (P_{s_t,a_t} V^*-V^*(s_t))$ can be bounded by $O(H^*\sqrt{T \log(1/\delta)})$ with probability $1-\delta$ using a standard martingale concentration inequality, where $H^*\le 2H/p$. 

The leading contributor to our regret bound is the third term in \eqref{eq:expectedRegretEquationMain} that requires bounding the sum of all bonuses $\sum_{t,h}  b_{n_{t,h}}$. We bound this sum by $\tilde{O}(H^4 S\sqrt{AT})$ using algebraic arguments similar to those used in \cite{jin2018q} for similar bonus terms. Compared to the episodic setting of \cite{jin2018q}, our bound on this term has a slightly higher dependence on $H,S$ because of the resetting of parameters in each epoch. 


On substituting these bounds into \eqref{eq:expectedRegretEquationMain} along with $C_1,C_2$, we obtain that with probability $1-\delta$,
\begin{equation}
\label{eq:reg1Main}
\textstyle
T\rho^* -  \frac{1}{H} \sum_{t,h} R(s_t,a_{t,h}) \le  \regBoundOtilde 
\end{equation}
To connect this bound to the definition of regret, recall that given state $s_t$ at time $t$, the algorithm samples $h_t$ uniformly at random from $\{1,\ldots, H\}$ and takes action $a_t$ as arg max of $Q^{t,h_t}(s_t,\cdot)$. Therefore, $a_t=a_{t,h_t}$, and we have
 $$\textstyle \Ex[\Reg(T)] = T\rho^* - \Ex[\sum_t R(s_t,a_{t,h_t})] = T\rho^* - \Ex[\frac{1}{H} \sum_{t,h} R(s_t,a_{t,h})]$$
 Therefore, the bound in  \eqref{eq:reg1Main} provides a bound on the expected regret of our algorithm. Since the per-step reward is bounded, the high probability regret bound of Theorem \ref{thm: main regret} can then be obtained by a simple application of the Azuma-Hoeffding inequality. 
 
All the missing details of this proof are provided in Appendix \ref{app:regret-analysis}.

\section{Conclusion}
We presented an optimistic Q-learning algorithm for online reinforcement learning under a setting that unifies episodic and average reward settings. Specifically, we consider MDPs with some (unknown) frequent state $s_0$ such that under all policies, the time to reach this state from any other state is upper bounded by a known constant $H$, either in expectation or with constant probability. 
 A main technical contribution of our work is to introduce an operator $\Lbar = \frac{1}{H}\sum_{h=1}^H L^h$ 
and demonstrate its strict span contraction property in our setting.
Using this property, we demonstrate a regret bound of $\tilde O(H^5 S\sqrt{AT})$ for our algorithm in the average reward setting, along with a corollary of $\tilde O(H^6 S\sqrt{AT})$ for the episodic setting.
An avenue for future research is to improve the dependence on $H$ in our regret bound. Such an improvement was not a focus of this work, but may be possible by employing techniques in some recent work on improving dependence on $H$ for episodic Q-learning, particularly \citep{li2021breaking}.

\bibliography{ref}
\bibliographystyle{plainnat}
\newpage 
\appendix


\begin{appendix}

\section{Preliminaries}
\label{app: preliminaries}
\subsection{Connection between the two assmptions}

\lemAssum*
\begin{proof} 
The first statement of this lemma follows simply from Markov inequality.

For the second statement, let Assumption \ref{assume: hitting time assumption H p} holds. Fix any $s\in \calS$. Starting from state $s$,  let $\tau$ be the time at which state $s_0$ is reached under a given Markovian policy (stationary or non-stationary). Let $s_1=s$, and $s_t$ be the state reached at time $t$. 
    \begin{eqnarray*}
     \Pr(\tau \ge t) & = & \Pr(s \rightarrow s_0 \ge t) \\
     & \le & \Pr(s_1,\ldots, s_{H\lfloor t/H\rfloor}\ne s_0) \\
        & = & \prod_{i=1}^{\lfloor t/H\rfloor} \Pr(s_{(i-1)H+1}, \ldots,  s_{iH} \ne s_0 | s_1,\ldots, s_{(i-1)H+1}) \\
        & = &  \prod_{i=1}^{\lfloor t/H\rfloor} \Pr(s_{(i-1)H+1} \rightarrow s_0\ge H| s_1,\ldots, s_{(i-1)H+1}) \\
        & \le & (1-p)^{\lfloor t/H\rfloor } 
        \end{eqnarray*}
Then, 
$$\Ex[\tau] = \sum_{t=1}^\infty \Pr(\tau\ge t) \le \sum_{t=1}^\infty (1-p)^{\lfloor t/H\rfloor } \le H \sum_{i=0}^\infty (1-p)^{i} \le \frac{H}{p}$$
Therefore, Assumption \ref{assume: expected hitting time assumption} holds with parameter $H$.
\end{proof}
\subsection{Unichain MDP and bounded span of bias vector}
\label{apx: weakly communicating}
\begin{lemma}
\label{lem: unichain}
    Under Assumption \ref{assume: expected hitting time assumption} and Assumption \ref{assume: hitting time assumption H p}, the MDP is unichain.
\end{lemma}
\begin{proof}
     Consider the Markov chain induced by any stationary policy in this MDP. Under either of these assumptions, there is a positive probability to go from $s$ to $s_0$ for every state $s$ in $n$ steps for some finite $n\ge 1$. Now, for any recurrent state $s$ there must be a positive probability of going from $s_0$ to $s$ in $n$ steps for some finite $n\ge 1$, otherwise the probability of revisiting $s$ would be strictly less than $1$. Therefore all recurrent states are reachable from each other and form a single irreducible class. All the remaining states are transient by definition. This proves that the MDP is unichain under either of these assumptions.
\end{proof}

\begin{lemma}
\label{lem: boundedBiasAllPolicies}
For any stationary policy $\pi$, define bias vector $V^\pi \in \mathbb{R}^S$ as follows, 
    $$ \textstyle V^\pi(s) = \lim_{T\to \infty}\frac{1}{T}\Ex\sbr{\sum_{t=1}^{T}(R(s_t,a_t)-\rho^\pi(s_t))|s_1=s; a_t=\pi(s_t)}.$$
    where $\rho^\pi(s) \coloneqq \lim_{T\to \infty}\frac{1}{T}\Ex\sbr{\sum_{t=1}^{T}R(s_t,a_t)|s_1=s; a_t=\pi(s_t)}$ is the asymptotic average reward of policy $\pi$ starting from state $s$. Then, under Assumption \ref{assume: expected hitting time assumption}, 
    the span of the vector $V^\pi$ is upper bounded by $2H$.
\end{lemma}
\begin{proof}     
    Let $J^\pi_n(s)$ be the $n$-step value of playing policy $\pi$  starting from the state $s$. Let state $s_0$ is reached in $\tau$ steps. Then, by our assumption $\Ex[\tau]\le H$. Therefore, since the reward in each time step is upper bounded by $1$, we have for every $s$,
     $$
     J^\pi_n(s_0) - H\le J^\pi_n(s) \leq J_n^\pi(s_0)+H,
     $$
     so that $\spa(J^\pi_n) \le 2H$. 
Since under Assumption \ref{assume: expected hitting time assumption} the MDP is unichain,  the gain of any stationary policy is constant (see Section 8.3.3 in \cite{puterman2014markov}), i.e., $\rho^{\pi}(s)=\rho^{\pi}(s')$ for all $s,s'$. This gives, by definition of $V^{\pi}$,
     $$V^\pi(s_1) - V^\pi(s_2) = \lim_{n\rightarrow \infty} J_n^\pi(s_1) - J_n^\pi(s_2).$$
     Therefore, $\spa(J^\pi_n) \le 2H$ for all $n$ implies $\spa(V^\pi)\le 2H$.
\end{proof}

\begin{lemma}
    \label{lem:spanVstar}
Under Assumption \ref{assume: expected hitting time assumption} and \ref{assume: hitting time assumption H p}, the span of the optimal bias/value vector $V^*$ is upper bounded by $2H$ and $2H/p$, respectively.
\end{lemma}
\begin{proof}
     Unichain MDPs are a special case of weakly communicating MDPs, for which the optimal policy is known to be stationary~\citep{puterman2014markov}. Therefore, applying Lemma \ref{lem: boundedBiasAllPolicies}, we obtain $\spa(V^*)\le 2H$ under Asumption \ref{assume: expected hitting time assumption}. And by Lemma \ref{lem: assumption connection} this implies $\spa(V^*)\le 2H/p$ under Assumption \ref{assume: hitting time assumption H p}. 
    \end{proof}

\subsection{Episodic MDP as a special case of Average reward MDP}
\label{app: epsiodic setting description}
Here we provide a reduction that shows that the regret minimization problem in episodic setting is a special case of \newToCheck{average reward setting} and satisfies Assumption \ref{assume: expected hitting time assumption} and \ref{assume: hitting time assumption H p}.

We define an episodic setting consistent with the recent literature (e.g., \cite{jin2018q,osband2013more}). 
In the episodic setting with finite horizon $H$, we have a time-inhomogeneous MDP described by the tuple $(\calS,\calA, P, R, H)$, where  $(P, R)= \{P^h,R^h\}_{h=1}^H$. 
 At each time step $h=1,\ldots, H$, the learner observes the state $s_h$, and takes an action $a_h \in \calA$. The MDP transitions to a new state $s_{h+1}\sim P^h_{s_h,a_h}$ and the agent receives a reward $R^h(s_h,a_h)$. 
 After $H$ steps, under any policy, the MDP reaches terminal state (say $s_0$) which is an absorbing state with reward $0$. 
Optimal policy aims to optimize the value function $V^{\pi}_h(s)$ at every step $h$ of the episode, defined as the $H-h+1$ step expected reward starting from state $s$ under (possibly non-stationary) policy $\pi=(\pi_1,\ldots, \pi_H)$:
$$
V^\pi_h(s) = \Ex[\sum_{j=h}^H R^h(s_j,a_j) | s_h=s; a_j\sim\pi_h(s_j)],
$$
where the expectation is taken over the sequence $s_{j+1} \sim P^j_{s_j,a_j},  a_j \sim \pi_h(s_j)$ for $j=h,\ldots, H$.  Then, by dynamic programming, the optimal value is given by 
$$
V^*_h(s) = \max_{\pi} V^\pi_h(s) = [LV^*_{h+1}](s) =  [L^{H-h+1} {\bf 0}](s)
$$
And, regret of an episode is defined as
$$ \textstyle V^*_1(s_1) - \sum_{h=1}^H R^h(s_h, a_h),$$
where $s_h, a_h$ are the state,action at step $h$ in the episode. Unlike the average reward setting, here the optimal reward in an episode depends on the starting state and the optimal policy is non-stationary. 

The {\it regret minimization problem in the episodic setting} seeks to minimize total regret over a horizon $T$, i.e. over $T/H$  episodes. That is, 
\begin{eqnarray*}
\Reg^{\mbox{\tiny{episodic}}}(T) & = &   \frac{T}{H} V^*_1(s_1) - \sum_{k=1}^{T/H} \sum_{h=1}^H R^h(s_{k,h},a_{k,h}) \\
\end{eqnarray*}
where $s_{k,h}, a_{k,h}$ denote the state, action at step $h$ in $k$th episode. 
\newToCheck{We show that this problem of regret minimization in episodic MDPs is in fact equivalent to the regret minimization problem in an average reward with a homogenous MDP that has slightly bigger state space ($HS$ states).} Specifically, we show the following result: 
\begin{lemma}
    Given any episodic MDP $M=(\calS, \calA, P, R, H)$, there exists a time-homogeneous MDP $M'=(\calS', \calA', P', R')$ satisfying Assumption \ref{assume: expected hitting time assumption} and \ref{assume: hitting time assumption H p} with $|\calS'|=H|\cal S|, |\calA'|=|\cal A|$ such that
    $$\Reg_{M}^{\mbox{\tiny{episodic}}}(T) = \Reg_{M'}(T).$$
    Here, $\Reg_{M}^{\mbox{\tiny{episodic}}}(T)$ and $\Reg_{M'}(T)$ denote the episodic and average reward regret under the MDPs $M$ and $M'$, respectively.  
\end{lemma}
\begin{proof}
    We prove this by the following simple reduction. Given episodic MDP $M=(\calS,\calA,P,R H)$, construct a (time-homogeneous) MDP $M'=(\calS',\calA,P',R')$ where corresponding to every state $s \in \calS$, the new state space $\calS'$ contains $H$ states denoted as $\{s^h,h=1,\ldots, H\}$. A visit to state $s$ at step $h$ in the episodic setting is then a visit to $s^h$ in MDP $M'$. And for all $s \in {\cal S},a\in {\cal A} ,h$, we define $P'(s^h,a) = P^h(s,a), R'(s^h, a) = R^h(s,a) $. Further, the transition model $P'$ is modified so that $s_0$ is not an absorbing state but instead transitions to starting state $s_1$ with probability $1$. Then any non-stationary policy in the episodic MDP $M$ is equivalent to a stationary policy $\pi'$ in the new MDP $M'$, with $\pi'(s^h) =\pi_h(s)$. Therefore, optimal non-stationary policy for the episodic MDP corresponds to a stationary policy for MDP $M'$. 
Because $s_0$ is visited every $H$ steps, the constructed MDP $M'$ trivially satisfies Assumption \ref{assume: expected hitting time assumption} and Assumption \ref{assume: hitting time assumption H p} (with $p=1$).  
Therefore, by Lemma \ref{lem: unichain} and Lemma \ref{lem:spanVstar}, $M'$ is unichain (and weakly communicating) with the span of optimal bias vector bounded by $H$.


Since $V^*_1(s_1)$ is the maximum reward obtainable per episode of fixed length $H$, clearly, the optimal asymptotic average reward $\rho^*$ for $M'$ is $\rho^* = \frac{1}{H} V^*_1(s_1)$. 
Construct vector $V^*$ as $V^*(s^h) =  V^*_{h}(s) - (H-h+1)\rho^*$; then we can show that the dynamic programming equation for the episodic MDP implies that the average reward Bellman optimality equations are satisfied by $\rho^*, V^*$. To see this  recall that we have by dynamic programming
$$V^*_1=LV^*_2 = \cdots = L^{H-1}V^*_H = L^H {\bf 0}$$
so that for every state $s^h$ in $M'$, we have 
\begin{eqnarray*}
    [LV^*](s^{h}) -V^*(s^{h}) & = & [LV^*_{h+1}](s) - (H-h)\rho^* -V^*_{h}(s) + (H-h+1)\rho^* \\
    & = & \rho^*
\end{eqnarray*}
Therefore, $\rho^*$ and $V^*$ are the optimal gain and bias vector (up to span), respectively, for MDP $M'$. 
Now, in the expression for $\Reg^{\mbox{\tiny{episodic}}}(T)$ above, substitute $\rho^*=V_1^*(s_1)/H$ and $R^h(s_{k,h}, a_{k,h})=R'(s_t,a_t)$ where $s_t$ is the state in $\calS'$ to corresponding  $s_{k,h}$, and $a_t=a_{k,h}$. We obtain
 \begin{eqnarray*}
\Reg^{\mbox{\tiny{episodic}}}_{M}(T) & = & T\rho^* - \sum_{t=1}^T R'(s_t,a_t) = \Reg_{M'}(T).
\end{eqnarray*}
\end{proof}


This discussion demonstrates that any algorithm constructed for our average reward setting under Assumption \ref{assume: expected hitting time assumption} or Assumption \ref{assume: hitting time assumption H p} can be seamlessly applied to the time in-homogeneous episodic MDP setting. Any regret bound obtained for our setting will hold almost as it is, with the only difference being the size of the state space which will change from $S$ to $SH$. 

\subsection{Span Contraction of $\Lbar$: Proof of Lemma \ref{lem: span contraction property}}\label{apx: span contraction} 
\begin{lemma}
\label{lem: span contraction property general}
 Define operator $\overline{L}:\mathbb{R}^S \rightarrow \mathbb{R}^S$ as: for any vector $v\in \mathbb{R}^S$, $\overline{L}v := \frac{1}{H} \sum_{h=1}^H L^hv.$ Then, given any 
$v_1,v_2\in \mathbb{R}^S$,
under Assumption~\ref{assume: hitting time assumption H p}, we have 
\begin{equation*}
   \textstyle \spa(\overline L v_1-\overline L v_2) \leq (1-\frac{p}{H})\spa( v_1 - v_2).
\end{equation*}
\end{lemma}

\begin{proof}
For this proof, we use the following notation. Given a stationary policy $\pi$, let $r_\pi, P_\pi$ denote the reward vector and transition matrix under this policy, i.e., $r_\pi(s):=\Ex_{a\sim \pi(s)}[R(s,a)], P_\pi(s,s'):=\Ex_{a\sim \pi(s)}[P_{s,a}(s')], \forall s,s'\in \calS$. And, let $\Pi$ denote the space of feasible policies.

First we prove the following statement by induction: for $i=1, 2, \ldots$
\begin{eqnarray}
\label{eq: span contraction induction statement 1 appendix}
L^i v_1 - L^{i} v_2 & \le&  \textstyle \left(\prod_{j=1}^i P_{\pi^1_j}\right) (v_1-v_2),\\
L^i v_1 - L^{i}v_2 & \ge &  \textstyle \left( \prod_{j=1}^i P_{\pi^2_j}\right) (v_1-v_2).\label{eq: span contraction induction statement 2 appendix}
\end{eqnarray} 
where $\pi^1_i \coloneqq\argmax_{\pi \in \Pi}r_{\pi} + P_{\pi}\cdot L^{i-1}v_1$, and $\pi_i^2\coloneqq\argmax_{\pi \in \Pi}r_{\pi} + P_{\pi}\cdot L^{i-1} v_2.$ 

For $i=1$, consider ${L}v_1-{L}v_2$: 
\begin{eqnarray*}
 Lv_1- Lv_2 &= & \argmax_{\pi \in \Pi}r_{\pi} + P_{\pi} v_1 - \sbr{\argmax_{\pi \in \Pi}r_{\pi} + P_{\pi} v_2} \\ &{\geq}&   P_{\pi^2_{1}}(v_1-v_2), 
\end{eqnarray*}
where in the last inequality we use $\pi^2_{1} \coloneqq \argmax_{\pi \in \Pi}r_{\pi} + P_{\pi} v_2$. For the upper bound, we have:
\begin{eqnarray*}
Lv_1 - Lv_2 &= & \argmax_{\pi \in \Pi}r_{\pi} + P_{\pi} v_1 - \sbr{\argmax_{\pi \in \Pi}r_{\pi} + P_{\pi} v_2} \\  &{\leq}&   P_{\pi^1_{1}}(v_1-v_2), 
\end{eqnarray*}
where in the last inequality we use $\pi^1_{1} \coloneqq\argmax_{\pi \in \Pi}r_{\pi} + P_{\pi} v_1$. 
Assume inequalities \eqref{eq: span contraction induction statement 1 appendix}, \eqref{eq: span contraction induction statement 2 appendix} are true for $i-1$, then we have:
\begin{eqnarray*}
L^iv_1 - L^iv_2 &=& L(L^{i-1}v_1) - L(L^{i-1} v_2)\\
&& \text{(applying the upper bound for $i=1$)}\\
& \le&  P_{\pi^1_1}( L^{i-1} v_1 - L^{i-1} v_2)\\
& &\text{(applying the upper bound for $i-1$)}\\
& \le&  P_{\pi^1_1} \cdots P_{\pi^1_i} ( v_1 - v_2),
\end{eqnarray*}
where $\pi^1_i \coloneqq\argmax_{\pi \in \Pi}r_{\pi} + P_{\pi}\cdot L^{i-1}v_1$.
Similarly, we can show the lower bound statement by induction.
\begin{eqnarray*}
L^iv_1 - L^i v_2 & =& L(L^{i-1}v_1) - L(L^{i-1} v_2)\\
&& \text{(applying the lower bound for $i=1$)}\\
& \ge&  P_{\pi^2_1}( L^{i-1} v_1 - L^{i-1} v_2)\\
&& \text{(applying the lower bound for $i-1$)}\\
& \ge&  P_{\pi^2_1} \cdots P_{\pi^2_i} ( v_1 - v_2),
\end{eqnarray*}
where $\pi_i^2\coloneqq\argmax_{\pi \in \Pi}r_{\pi} + P_{\pi}\cdot L^{i-1} v_2.$ 

This completes the proof of inequalities \eqref{eq: span contraction induction statement 1 appendix},\eqref{eq: span contraction induction statement 2 appendix}.
We have by Assumption~\ref{assume: hitting time assumption H p}, for all starting state distributions $\mu$ over states in ${\cal S}$, 
 $$ \textstyle \mu^T (\sum_{i=1}^{H} P_{\pi^1_1} \cdots P_{\pi^1_i} ) \ge p \one_{s_0}^T,\ \  \mu^T(\sum_{i=1}^{H}   P_{\pi^2_1} \cdots P_{\pi^2_i})   \ge  p\one_{s_0}^T.$$
In particular using above with $\mu=\one_s$, the Dirac delta distribution for state $s$, and substituting the inequality  \eqref{eq: span contraction induction statement 1 appendix}, we have for all $s$,
\begin{eqnarray*}
 \sum_{i=1}^H \left[L^i v_1 - L^i v_2\right](s) & = & \textstyle \one_s^T  \left(\sum_{i=1}^H (L^i v_1 - L^i v_2)\right)\\
 & \le &  \textstyle  \one_s^T \left(\sum_{i=1}^{H} P_{\pi^1_1} \cdots P_{\pi^1_i}\right)(v_1-v_2) \\
 & \le & p \one_{s_0}^T (v_1-v_2) + (H-p)   \max_{s'} (v_1(s')-v_2(s') \\
 & = & p(v_1(s_0) - v_2(s_0)) + (H-p) \max_{s'} \{v_1(s')-v_2(s')\},\\
 \text{so that } & &\\
 \max_s \sum_{i=1}^H \left[L^i v_1 - L^i v_2\right](s) & \le & p(v_1(s_0) - v_2(s_0)) + (H-p) \max_{s'} \{v_1(s')-v_2(s')\}.
\end{eqnarray*} 
Similarly, substituting the inequality \eqref{eq: span contraction induction statement 2 appendix}
\begin{eqnarray*}
 \sum_{i=1}^H \left[L^i v_1 - L^i v_2\right](s) & = & \textstyle \one_s^T  \left(\sum_{i=1}^H (L^i v_1 - L^i v_2)\right)\\
 & \ge &  \textstyle \one_s^T \left(\sum_{i=1}^{H} P_{\pi^2_1} \cdots P_{\pi^2_i}\right)(v_1-v_2) \\
 & \ge & p \one_{s_0}^T (v_1-v_2) + (H-p)   \min_{s'} (v_1(s')-v_2(s') \\
 & = & p(v_1(s_0) - v_2(s_0)) + (H-p) \min_{s'} \{v_1(s')-v_2(s')\}\\
 \text{so that } &&\\
  \min_s \sum_{i=1}^H \left[L^i v_1 - L^i v_2\right](s) & \ge & p(v_1(s_0) - v_2(s_0)) + (H-p) \min_{s'} \{v_1(s')-v_2(s')\}.
\end{eqnarray*} 
Therefore, subtracting the two inequalities we get,
$$\textstyle \spa\left(\frac{1}{H} \sum_{i=1}^H (L^i v_1 - L^iv_2)\right) \le  \left(1-\frac p H \right) \spa(v_1-v_2).$$
\end{proof}

As a corollary to Lemma~\ref{lem: span contraction property general}, we obtain Lemma \ref{lem: span contraction property}, stated again here for easy reference.
\spanContractionCorollary*
\begin{proof}
We use that for any vectors $v_1,v_2, v$, $\spa(v_1+v_2)\leq \spa(v_1)+\spa(v_2)$ and $\spa(c v) = c \cdot \spa(v)$, to get,
$$\textstyle \spa\left(\overline{L} v -V^*\right) \le \spa\left(\overline{L} v - \overline{L} V^*\right) + \frac{1}{H} \sum_{i=1}^H \spa\left(L^i V^* - V^*\right),$$
Then, substituting the result from Lemma \ref{lem: span contraction property general} for the first term, along with  the observation that 
$$\spa(L^iV^*-V^*)\le \sum_{j=1}^i \spa(L^jV^*-L^{j-1} V^*)\le \sum_{j=1}^i \spa(LV^*-V^*) = 0, $$ we get the lemma statement. In the last inequality above we applied the span contraction property of the Bellman operator $L$: $\spa(Lv_1-Lv_2)\le \spa(v_1-v_2)$ for any two vectors $v_1,v_2$. 
\end{proof}
\subsection{Comparison of $\Lbar$-contraction to $J-$stage contraction in~\cite{puterman2014markov}}\label{app: putterman contraction operator comparison}
(8.5.8) in~\cite{puterman2014markov} defines the $J$-stage contraction property as strict contraction of $L^J$ with $L$ being the standard Bellman operator with discount factor $1$. Even though on the surface this appears to be similar to the strict contraction property of the operator $\Lbar$, 
the assumptions required for such $J$-stage contraction can in fact be significantly more restrictive. In particular, the sufficient condition provided in the theorem 8.5.2 of~\cite{puterman2014markov} for this kind of contraction does not hold under either of our Assumption~\ref{assume: expected hitting time assumption} or~\ref{assume: hitting time assumption H p}, and $\eta=(1-\gamma)$ could be 0.

More intuitively, the closest assumption to ours under which the $J$-stage contraction would hold is if there were a lower bound of $p$ on the probability of reaching $s_0$ exactly at the same time $H$ starting from any state, under all policies. Instead, our operator $\Lbar$ which is defined as the average of $L^h$ for $h=1,\ldots, H$, requires only an upper bound on the time to reach $s_0$. This distinction can be very important, as we illustrated using the following toy example.

Consider an MDP with 3 states where $s_0$, $s_1$, and $s_2$ and two policies. Under the first policy, $s_1$ has probability $1$ to reach $s_0$, and $s_2$ has probability $1$ to reach $s_1$. Under the second policy, the roles of $s_1$ and $s_2$ are reversed. In this MDP, the sufficient condition for $J$-stage contraction is not satisfied ($\gamma =1$ in Theorem 8.5.2 of~\cite{puterman2014markov}), since for different states the time to reach $s_0$ is different under either policy. But our Assumption~\ref{assume: hitting time assumption H p} is satisfied with an upper bound $H=2$ on time to reach $s_0$ with probability $p=1$, giving a strict $\Lbar$-contraction with contraction factor $1/2$.

\section{Missing Proofs from Section \ref{sec: regret analysis}}
\label{apx: regret analysis}

\renewcommand{\VinitConstH}{H}
\subsection{Optimism: Proof of Lemma \ref{lem:optimismNew}}

\optimismNew*
\begin{proof}

We prove $H+1$ inductive statements $\is_h(\ell),h=1,\ldots, H+1$ defined as follows.
\begin{itemize}
    \item ${\cal I}_{H+1}(\ell)$ states the inequality  \eqref{eq:induction1} holds for epoch $\ell$ and all $k\ge 0$ such that $\ell-k=Kj+1$ for some integer $0\le j\le \jvalue$. 
    \item For each $h\in [H]$, ${\cal I}_{h}(\ell)$  states the $h_{th}$ inequality in \eqref{eq:induction2} holds for all $t\in \ep \ell$ and all $k\ge 0$ such that $\ell-k=Kj+1$ for some integer $0\le j\le \jvalue$ 
\end{itemize}
To prove the lemma, we show that $\is_{h}(\ell)$ holds for all $h=1,\ldots, H+1$ and all $\ell$ with probability at least $1-\delta$. 

Our induction-based proof works as follows. We prove $\is_{H+1}(1)$ as the base case. Next, we prove $\is_{h}(\ell)$ for any $h\le H$ holds with probability $1-\frac{\delta}{HT}$, assuming $\is_{h'}(\ell')$  holds 
for all $\ell'\le \ell-1, h'\in [H+1]$ and $\ell'=\ell, h'\ge h+1$. And then we prove $\is_{H+1}(\ell)$ holds assuming statements $\is_h(\ell')$ hold for all $\ell'\le \ell-1$ and $h=1,\ldots, H+1$. 
Since the total number of epochs is at most $T$, taking a union bound over $\ell,h$ will give that the statements $\is_h(\ell)$ hold for all $h\in 1,\ldots, H+1,\ell$ with probability at least $1-\delta$. \newline\\
\noindent{\underline{Base case $\is_{H+1}(1)$:}} Consider $\ell=1$,  then $k=0$. The first inequality \eqref{eq:induction1} reduces to  $\Vbar^1 \ge \Vbar^1$, which is trivially true. \newline\\

\noindent{\underline{Induction step for $\is_h(\ell), h\in [H]$.}} 
Assume $\is_{h'}(\ell')$  hold for all $\ell'\le \ell-1, h'\in [H+1]$ and $\ell'=\ell, h'\ge h+1$. We show that $\is_h(\ell)$ will hold with probability at least $1-\frac{\delta}{HT}$.

Let $Q^{t,h}, V^{t,h}, N^{t}$ denote the value of $Q^{h}, V^h, N$ {\it at the beginning} of time step $t$ in epoch $\ell$ (i.e., not counting the sample at time $t$). 

Fix a $t\in \ep \ell$ and $s\in \calS$. First, let us consider the case that $N^t(s,a)=0$ for some $a$. 
In this case $Q^{t,h}(s,a)$ takes its initial value  $\max_{s'} \Vbar^\ell(s')+\VinitConstH$, so that $V^{t,h}(s)= \max_{a'} Q^{t, h}(s,a') \ge  Q^{t,h}(s,a) =\max_{s'} \Vbar^\ell(s')+\VinitConstH$. Then, using the induction hypothesis for $\Vbar^\ell(s)$ (i.e., $\is_{H+1}(\ell)$), we have, 
\begin{eqnarray}
\label{eq:lbInit}
    V^{t, h}(s)  \ge \VinitConstH + \max_{s'} \Vbar^{\ell}(s')
    & \ge & \VinitConstH +\max_{s'} \Lbar^{k}\Vbar^{\ell-k}(s')
    \ge  [L^{H-h+1}\Lbar^{k} \Vbar^{\ell-k}](s)
\end{eqnarray}
where in the last inequality we used that 
each $L$ operator can add at most $1$ to the max value of the vector. 

Let us now show the lower bound for the case $N^t(s,a)\ge 1$ for all $a$. By algorithm design, after the update, for any $s,a$ with $N^t(s,a)\ge 1$, the $Q$ estimate available at the beginning of time step $t$ for $h=1,\ldots, H$ is
\begin{eqnarray*}
Q^{t,h}(s,a)&= & \sum_{i=1}^n \alpha_n^i(r_{t_i}+V^{t_i+1,h+1}(s_{t_i+1}) + b_i),
\end{eqnarray*}
where $n=N^{t}(s,a)$, and $\sum_{i=1}^n \alpha_n^i=1$.
Here $V^{t_i+1,H+1}$ stays at its initial value of $\Vbar^\ell$ set in the beginning of the epoch. 

Then (with some abuse of notation, in below we use $n\coloneqq N^{t}(s_t,a)$ where identity of action $a$ is clear from the context). 

\begingroup
\allowdisplaybreaks
\begin{eqnarray*}
V^{t,h}(s)  & = & \max_a Q^{t,h}(s,a) \nonumber\\
  & = & \max_a \left(\sum_{i=1}^n \alpha_{n}^i(r_{t_i}+V^{t_i+1,h+1}(s_{t_i+1}) + b_i) \right), \text{ where $n= N^t(s,a)$}\nonumber\\
  && \text{below we use hypothesis $\is_{h+1}(\ell)$ for $h+1\le H$, and for $h+1=H+1$, we use that}\nonumber\\
&& \text{ $V^{t_i+1,H+1} = \Vbar^\ell \ge \Lbar^{k}\Vbar^{\ell-k}$ using  inductive statement $\is_{H+1}(\ell)$)}\nonumber \\
  & \ge & \max_a \left(\sum_{i=1}^n \alpha_{n}^i(r_{t_i}+ L^{H-h}\Lbar^{k}\Vbar^{\ell-k} (s_{t_i+1}) + b_i) \right), \text{ where $n= N^t(s,a)$}\nonumber\\
    & & \text{(since $s_{t_i}=s$, using Lemma~\ref{lem: concentration based on Chi jin} with  $\sigma\le 4H^*$})\\
    && \text{for bounding reward terms and transition terms, we get with probability }\textstyle 1-\frac{\delta}{HT^3},) \\
        & \ge & \max_a \left(\sum_{i=1}^{n} \alpha_{n}^i (R(s,a) + P_{s, a}\cdot  L^{H-h}\Lbar^{k}\Vbar^{\ell-k} + b_i) - b_n  \right)\nonumber \\
        & & \text{(using $b_i \ge b_n$ for $i\ge 1$, and  $\sum_{i=1}^n\alpha^i_n=1$ since $n\ge 1$)}\nonumber\\
 &\ge & \max_a \left( R(s,a)+P_{s,a} \cdot L^{H-h}\Lbar^{k}\Vbar^{\ell-k}\right)\\
 && \text{(by definition of $L(\cdot)$)}\nonumber\\
 &= & [L^{H-h+1}\Lbar^{k}\Vbar^{\ell-k}](s). 
\end{eqnarray*}
\endgroup
For applying the concentration bound from Lemma \ref{lem: concentration based on Chi jin} in above, we bound $\sigma$ by the span of $L^{H-h}\Lbar^k\Vbar^{\ell-k}$, which is bounded by $4H^*$ by Lemma \ref{lem: trivial span bound KH} using that $\ell-k=Kj+1$ for some integer $j\ge 0$.
This bounded span property holds due to the projection operator using which the algorithm occasionally clips the span of $\Vbar$. We prove this in Lemma \ref{lem: trivial span bound KH}.  

Above inequality holds with probability at least $1-\frac{\delta}{HT^3}$ for each $k$ and $t\in \ep \ell$, and all $s$. (Since Lemma \ref{lem: concentration based on Chi jin} holds for all $s,a$ simultaneously, we do not need to take union bound over $s,a$.)
Taking union bound over $t$ and $k$ we have that $\is_h(\ell)$ holds with probability at least $1-\frac{\delta}{HT}$.\newline\\

\noindent{\underline{Induction step for $\is_{H+1}(\ell)$.}} Assume $\is_{h}(\ell')$ holds for $h=1,\ldots H+1, \ell'\le \ell-1$. We prove that $\is_{H+1}(\ell)$ holds.


First consider the case when $N_{\ell-1}(s)=0$. In this case  $\Vbar^{\ell}(s)$ did not  get any updates during epoch $\ell-1$, and will take the value as initialized in the beginning of the epoch $\ell-1$, that is, $\Vbar^{\ell}(s)=\max_{s'} \Vbar^{\ell-1}(s')+\VinitConstH$. 
Then, using induction hypothesis $\is_{H+1}(\ell-1)$,
\begin{eqnarray}
\Vbar^{\ell}(s) = \max_{s'} \Vbar^{\ell-1}(s')+\VinitConstH
     \ge  \max_{s'} \Lbar^{k-1} \Vbar^{\ell-1-(k-1)}(s') +H \ge [\Lbar^{k} \Vbar^{\ell-k}](s),
\end{eqnarray}
where we can apply the induction hypothesis for $\ell-1, k-1$ since $\ell-1-(k-1)=\ell-k$ is of the form $Kj+1$ for some feasible integer $j$.

We can now restrict to the case when $N_{\ell-1}(s)\ge 1$. 
For  such $s$, and $\ell\ge 2$, we have,
$$v^\ell(s) :=\frac{1}{N_{\ell-1}(s)}\sum_{t\in \text{ epoch }\ell-1: s_{t}=s} \frac{1}{H} \sum_{h=1}^H V^{t,h}(s).$$
And, by algorithm construction, we have, 
\begin{eqnarray}
    \Vbar^{\ell}(s) & = & \left\{\begin{array}{rl}
         [\spaProj v^{\ell}](s), & \text{ if } (\ell-1 \mod K)=0,\\
         v^{\ell}(s), & \text{otherwise}. 
    \end{array}
    \right.
\end{eqnarray}
Fix any $\ell\ge 2$. First consider $\ell$ such that $(\ell-1)\mod K \ne 0$ for some integer $m>0$. For such $\ell$, $\Vbar^{\ell}=v^\ell$ as defined above, and using the induction hypothesis $\is_h(\ell-1)$ for $h=1,\ldots, H$ and for $k-1$ (so that $\ell-1-(k-1) = \ell-k = Kj+1$ for a feasible integer $j$) we have
\begin{eqnarray}\label{eq: optimism proof second induction step lower bound helper 1 large ell}
\Vbar^{\ell}(s) = v^{\ell}(s) & = & \frac{1}{N_{\ell-1}(s)}\sum_{t\in \text{ epoch }\ell-1: s_{t}=s} \frac{1}{H} \sum_{h=1}^H V^{t,h}(s)\nonumber\\
& \ge &  \frac{1}{N_{\ell-1}(s)}\sum_{t\in \ep \ell-1: s_{t}=s} \frac{1}{H} \sum_{h=1}^H L^{H-h+1}\overline{L}^{k-1} \Vbar^{\ell-1-(k-1)}(s) \nonumber\\
& = &  \frac{1}{H} \sum_{h=1}^H L^{H-h+1}\overline{L}^{k-1} \Vbar^{\ell-k}(s) \nonumber\\
    &=& \left[\overline{L}^{k} \Vbar^{\ell-k}\right](s).
\end{eqnarray}

Now, consider the case when $(\ell-1)\mod K=0, \ell \ge 2$. (Note that in that case $\ell\ge K+1$). In this case 
$\Vbar^{\ell}=\spaProj v^\ell$. First, we use the monotonicity property $\spaProj v \ge \spaProj u, v\ge u$ (refer to Lemma \ref{lem: span projection}) along with the above lower bound for $v^{\ell}$ to get 
$$\Vbar^\ell = \spaProj v^\ell \ge \spaProj \overline{L}^{k} \Vbar^{\ell-k}$$

Next we use the property that when $\spa(v) \le 2H^*$, $\spaProj v = v$ (see Lemma \ref{lem: span projection} (c,d)). 
To apply this property, note that  we are only considering $k$  of form $k=\ell-(Kj+1)$ where $j\le \max(0, \lfloor \frac{\ell-K-1}{K}\rfloor)$ (see lemma statement). 
If $j=0$, then $k=\ell-1 \ge K$, and otherwise $j\le \lfloor \frac{\ell-K-1}{K}\rfloor$ so that $k \ge \ell-(Kj+1)\ge K$. In either case, since $k\ge K$, we can use the span bound in Lemma \ref{lem: trivial span bound KH} to obtain $\spa(\overline{L}^{k} \Vbar^{\ell-k})\le 2H^*$. 
And therefore, 
$$\Vbar^{\ell} \ge   \spaProj (\overline{L}^{k}) \Vbar^{\ell-k} = \overline{L}^{k} \Vbar^{\ell-k}.$$
This finishes the proof of inductive statement $\is_{H+1}(\ell)$.
\end{proof}
\subsection{Upper bound on Q-values: Proof of Lemma \ref{lem:Q-upper-bound}}

In this section, we prove the aggregate upper bound on Q-values given by Lemma \ref{lem:Q-upper-bound}. 

To prove this lemma, first, we prove a per-step upper bounds on $\Vbar^{\ell}$ and $Q^{t,h}$ for any epoch $\ell$ and any $t,h$. These bounds are derived in terms of two recursively defined quantities $G^{k}(\ell,s)$ and $g^k(t,h)$, respectively, for any $k\ge 1$. These quantities capture the recursive expression of accumulated errors in these estimates. 
\newcommand{\bound}[1]{5#1Hb_0}
\newcommand{\ginit}[1]{\bound{#1}+5(H+1-h)b_0}
\newcommand{\Ginit}[1]{\bound{#1}}

We recall some notation defined elsewhere. We use  $n_{t,h}:=N^{t}(s_{t},a_{t,h})$ where $a_{t,h}:=\arg \max_a Q^{t,h}(s_{t},a)$. For a given $s,a$, $t_i$  denotes the time step of $i^{th}$ occurrence of $s,a$. For a time step $t$, $\ell_t$ denotes the epoch in which time step $t$ occurs. And for an epoch $\ell$, $N_\ell(s)$ for any state $s$ denotes the number of visits of state $s$ in this epoch.
\begin{itemize}
\item 
$g^k(t,h)$ is defined as follows: $g^0(t,h):=0, \forall t$. For $k\ge 1$, 
$$g^k(t,H+1):= G^{k-1}(\ell_t,s_t);$$
and for $h=1,\ldots, H$, if $n_{t,h}\ge 1$, we define
$$\textstyle g^k(t,h):= b_{n_{t,h}}+\sum_{i=1}^{n_{t,h}} \alpha^i_{n_{t,h}}g^{k}(t_i+1,h+1),$$
 otherwise if $n_{t,h}=0$, we define $g^k(t,h):= \ginit{(k-1)}$.
\item For any epoch $\ell$ and state $s$, $G^k(\ell,s)$ is defined as follows:  $G^0(\ell, s):=0, \forall \ell, s$; and for $k \ge 1, N_{\ell-1}(s)\ge 1$,
$$ G^{k}(\ell,s):=\frac{1}{N_{\ell-1}(s)}\sum\limits_{t\in \text{epoch }\ell-1: s_{t}=s} \frac{1}{H} \sum_{h=1}^H g^{k}(t,h), $$
\vspace{-0.1in}
For $s,\ell$ with $N_{\ell-1}(s)=0$, we set $G^{k}(\ell,s)=\Ginit{k}$. 
\end{itemize}

 Given these definitions, we observe the following bounds on $g^k(t,h)$ and $G^k(\ell,s)$ which will be useful later. 
 \begin{lemma} 
 \label{lem:upperBoundsgG}
 For all $\ell,t,h=1,\ldots, H,s$, and $k\ge 0$
 \begin{itemize}
     \item $G^k(\ell,s)\le \Ginit{k}$
     \item $g^{k+1}(t,h)\le \ginit{k} \le \Ginit{(k+1)}$
 \end{itemize} 
 \end{lemma}
 \begin{proof}
 We prove this by induction on $k$ and $h$. It is clearly true for $G^0(\ell,s)$ since these quantities are initialized as $0$. Then, given any $k\ge 0$, assume that it is true for $G^{k}(\cdot,\cdot)$. 
  Then, for all $t$, $g^{k+1}(t,H+1)=G^k(\ell_t,s_t)\le \Ginit{k}$. Now, for any $t,h\le H$, assume $g^{k+1}(t,h+1) \le \Ginit{k} + 5(H-h)b_0$, then, either $n_{t,h}=0$ in which case $g^{k+1}(t,h)=\ginit{k}$, or $n_{t,h}>0$, in which case $g^{k+1}(t,h)=b_{n_{t,h}}+ \sum_i \alpha_{n_{t,h}}^i g^{k+1}(t_i+1,h+1) \le b_0 + \Ginit{k}+5(H-h)b_0 = \ginit{k}$.
Finally,  $G^{k+1}(\ell,s)\le \Ginit{(k+1)}$ since it is either initialized as $\Ginit{(k+1)}$ or is an average of $g^{k+1}(t,h)$ over some $t,h$. This completes the induction.
 \end{proof}

 Next, we prove the following per-step upper bound on value and $Q$-value estimates.
\begin{lemma}
\label{lem: Q-upper-bound-per-step}
Let $a_{t,h}$ denote the arg max action for $Q^{t,h}(s_t,\cdot)$, i.e., $a_{t,h}:=\arg \max_{a} Q^{t,h}(s_t,a)$. Then, with probability at least $1-\delta/2$, we have for all epochs $\ell$, $t\in \ep \ell$,  and $h=1,\ldots, H$,
\begin{eqnarray}
\Vbar^{\ell}(s_t) & \le & [\overline{L}^{k} V^{\ell-k}](s_t)  + 4 G^{k}(\ell,s_t),\label{eq:induction3}\\
V^{t,h}(s_t)=Q^{t,h}(s_t,a_{t,h}) & \le & \left(R(s_t,a_{t,h}) + P_{s_t,a_{t,h}} \cdot L^{H-h}\overline{L}^{k} \Vbar^{\ell-k}\right) + 4g^{k+1}(t,h), 
\label{eq:induction4}
\end{eqnarray}
for any $0\le k\le \ell-1$ such that $\ell-k=Kj+1$ for some integer $j$.
\end{lemma}
\begin{proof}
  We prove $H+1$ induction statements $\is_h(\ell),h=1,\ldots, H+1$. 
\begin{itemize}
    \item ${\cal I}_{H+1}(\ell)$ states the inequality  \eqref{eq:induction3} holds for all $t\in \ep \ell$ and all $0\le k\le \ell-1$ such that $\ell-k=Kj+1$ for some integer $j\ge 0$. 
    \item For each $h=1,\ldots, H$, ${\cal I}_{h}(\ell)$  states the $h^{th}$ inequality in \eqref{eq:induction4} holds for all $t\in \ep \ell$ and all $0\le k\le \ell-1$ such that $\ell-k=Kj+1$ for some integer $j\ge 0$. 
\end{itemize}
To prove the lemma, we show that $\is_{h}(\ell)$ holds for all $h=1,\ldots, H+1$ and all $\ell$ with probability at least $1-\delta$. 

Our induction-based proof works as follows. We prove $\is_{H+1}(1)$ as the base case. Next, for each $h=1,\ldots, H$, we prove $\is_{h}(\ell)$ holds with probability $1-\frac{\delta}{2HT}$, assuming $\is_{h'}(\ell')$  holds 
for all $\ell'\le \ell-1, h'\in [H+1]$ and $\ell'=\ell, h'\ge h+1$. And then we prove $\is_{H+1}(\ell)$ holds assuming statements $\is_h(\ell')$ hold for all $h=1,\ldots, H+1$ and $\ell'\le \ell-1$. 
Then since the total number of epochs is at most $T$, taking a union bound over all $\ell, h$ will give that the statement $\is_h(\ell), h=1,\ldots, H+1$ hold for all epochs $\ell$ with probability at least $1-\delta/2$. \newline\\
\noindent{\underline{Base case $\is_{H+1}(1)$:}} Consider $\ell=1$,  then $k=0$. The first  inequality \eqref{eq:induction3} reduces to  $\Vbar^1 \le \Vbar^1$, which is trivially true. \newline\\
\noindent{\underline{Induction step for $\is_h(\ell), h\in [H]$.}} 
Fix an $h\in [H]$. Assume $\is_{h'}(\ell')$  holds 
for all $\ell'\le \ell-1, h'\in [H+1]$ and for $\ell'=\ell, h'\ge h+1$. We show that $\is_h(\ell)$ holds with probability $1-\frac{\delta}{2HT}$.

Let $Q^{t,h}(s,a), V^{t,h}(s), N^{t}(s)$ denote the value of $Q^{h}(s,a), V^h(s), N(s)$ {\it at the beginning} of time step $t$ in epoch $\ell$ (i.e., not counting the sample at time $t$) in Algorithm \ref{alg:main}. 

Fix a $t\in \ep \ell$. First, let us consider the case that $n_{t,h}:=N^t(s_t,a_{t,h})=0$. 
In this case $Q^{t,h}(s_t,a_{t,h})$ takes its initial value \red{note that I changed the initialization to max and VinitConstH value to $H$} $\max_{s'} \Vbar^\ell(s')+\VinitConstH$  and (see definition in the beginning of this section) we have $g^{k+1}(t,h)=\ginit{k}$. Then, using the induction hypothesis for $\Vbar^\ell(s)$, the following can be derived for any $0\le k\le \ell-1$ such that $\ell-k=Kj+1$ for any integer $j$ satisfying the conditions stated in the lemma.
\begin{eqnarray*}
    Q^{t, h}(s_t,a_{t,h}) & = & \max_{s'} \Vbar^\ell(s')+\VinitConstH \\
    & \le &  \max_{s'} \left(\Lbar^k\Vbar^{\ell-k}(s') + 4 G^k(\ell,s') \right)+\VinitConstH \\
    & \le &  \max_{s'} \left(\Lbar^k\Vbar^{\ell-k}(s')\right) + \max_{s'} 4 G^k(\ell,s')+\VinitConstH \\
    && \text{(we use that since $\ell-k=Kj+1$, $\spa(L^{H-h+1}\Lbar^k\Vbar^{\ell-k}) \le 4H^*$  by Lemma \ref{lem: trivial span bound KH})} \\
    & \le &  \min_{s'} \left(\Lbar^k\Vbar^{\ell-k}(s')\right)+  4H^*  +4 \max_{s'} G^k(\ell,s')+\VinitConstH\\
    && \text{(we use that since reward is non-negative, the $L$ operator can only increase the min value)}\\
    & \le &  \min_{s'} \left(L^{H-h+1}\Lbar^k\Vbar^{\ell-k}(s')\right)+4H^*+ 4 \max_{s'} G^k(\ell,s')+\VinitConstH \\
    & \le  & L^{H-h+1}\Lbar^k\Vbar^{\ell-k}(s_t)+ 4 g^{k+1}(t,h),
\end{eqnarray*}
where in the last inequality we used the upper bound $G^k(\ell,s)\le \Ginit{k}$ derived in Lemma \ref{lem:upperBoundsgG}, and that $g^{k+1}(t,h)=\ginit{k}$  in this case.

Now, restrict to the case with $n_{t,h}\ge 1$. For $h$ such that $h\le H-1$, we can apply induction hypothesis $\is_{h+1}(\ell)$ (i.e., inequality  \eqref{eq:induction4} for $h+1$) to obtain the following upper bound on $V^{t,h+1}$ for any $t\in \ep \ell$. 
\begin{eqnarray*}
    V^{t,h+1}(s_{t}) & = &  Q^{t,h+1}(s_t, a_{t,h+1}) \le \left(R(s_t,a_{t,h}) + P_{s_t,a_{t,h}} \cdot L^{H-h-1}\overline{L}^{k} \Vbar^{\ell-k}\right) + 4g^{k+1}(t,h+1), \\
& \le & L^{H-h}\overline{L}^{k}\Vbar^{\ell-k}(s_t)+ 4 g^{k+1}(t,h+1). 
\end{eqnarray*}
And for $h=H$, we have that since $V^{t,h+1}=V^{t,H+1}$ is never updated and remains at its initial value $\Vbar^\ell$, applying induction hypotheses $\is_{H+1}(\ell)$,
$$V^{t,h+1}(s_{t}) =  \Vbar^\ell(s_t) \le \overline{L}^{k}\Vbar^{\ell-k}(s_t)+ 4 G^{k}(\ell,s) = L^{H-h}\overline{L}^{k}\Vbar^{\ell-k}(s_t)+ 4 g^{k+1}(t,h+1).$$
where we used that $g^{k+1}(t,H+1)=G^k(\ell,s_t)$.
We use this to derive, 
\begingroup
\allowdisplaybreaks
\begin{eqnarray}\label{eq: optimism proof upper bound 2}
V^{t,h}(s_{t}) & = &  Q^{t,h}(s_t, a_{t,h})\nonumber \\
& = & \sum_{i=1}^{n_{t,h}} \alpha_{n_{t,h}}^i(r_{t_i}+V^{t_i+1,h+1}(s_{t_i+1}) + b_i) \nonumber\\ 
& & \text{ (using the upper bound derived above for $V^{t_i+1,h+1}(s_{t_i+1})$)}\\
& \le & \sum_{i=1}^{n_{t,h}} \alpha_{n_{t,h}}^i\left(r_{t_i}+ [L^{H-h}\overline{L}^{k}\Vbar^{\ell-k}](s_{t_i+1}) + 4g^{k+1}(t_i+1,h+1) +b_i\right)  \nonumber\\
& & \text{(using $\sum_{i=1}^{n} \alpha^i_{n}b_i \le 2b_n$, and Lemma~\ref{lem: concentration based on Chi jin} with $\sigma\le 4H^*$, we have with probability $\textstyle 1-\frac{\delta}{4T^3 H}$,)} \nonumber\\
& \le& R(s_{t},a_{t,h})+P(s_{t},a_{t,h})L^{H-h}\overline{L}^k\Vbar^{\ell-k}+4\left(\underbrace{b_{n_{t,h}}+\sum_{i=1}^{n_{t,h}} \alpha_{n_{t,h}}^i g^{k+1}(t_i+1,h+1)}_{= g^{k+1}(t,h)}\right)  \nonumber
\end{eqnarray}
\endgroup
Therefore, we have that \eqref{eq:induction4} holds for all permissible $k$ and for all $t$ in epoch $\ell$ with probability at least $1-\frac{\delta}{2T^3 H}$. Taking union bound over $k$ and $t$ we have that $\is_h(\ell)$ holds with probability at least $1-\frac{\delta}{2TH}$.\newline\\

\noindent{\underline{Induction step for $\is_{H+1}(\ell)$, $\ell\ge 2$.}} Assume $\is_{h}(\ell')$ holds for $h=1,\ldots H+1, \ell'\le \ell-1$. We show that $\is_{H+1}(\ell)$ holds.


First consider the case when $N_{\ell-1}(s)=0$. In this case  $\Vbar^{\ell}(s)$ did not  get any updates during epoch $\ell-1$, and will take the value as initialized in the beginning of the epoch $\ell-1$, that is, $\Vbar^{\ell}(s)=\max_{s'} \Vbar^{\ell-1}(s')+\VinitConstH$ and $G^k(\ell,s)=\Ginit{k}$. 
Then, using induction hypothesis $\is_{H+1}(\ell-1)$,
\begin{eqnarray*}
   \Vbar^\ell(s) & = & \max_{s'} \Vbar^{\ell-1}(s')+\VinitConstH \\
    & \le &  \max_{s'} \left(\Lbar^{k-1}\Vbar^{\ell-1-(k-1)}(s') + 4 G^{k-1}(\ell,s') \right)+\VinitConstH \\
    &&\text{we use that since reward is non-negative, the $L$ operator can only increase the max value}\\
    & \le &  \max_{s'} \left(\Lbar^{k}\Vbar^{\ell-1-(k-1)}(s')\right)+ 4\max_{s'} G^{k-1}(\ell,s') +\VinitConstH \\
     && \text{we use that since $\ell-k=Kj+1$, by Lemma \ref{lem: trivial span bound KH}, $\spa(\Lbar^k\Vbar^{\ell-k}) \le 4H^*$} \\
    & \le &  \min_{s'} \left(\Lbar^k\Vbar^{\ell-k}(s')\right)+  4H^* +4\max_{s'} G^{k-1}(\ell,s') +\VinitConstH \\
     & \le &  \min_{s'} \left(\Lbar^k\Vbar^{\ell-k}(s')\right)+  4H^* + 4(\Ginit{(k-1)})+\VinitConstH \\
    & \le  & \Lbar^k\Vbar^{\ell-k}(s)+ 4 G^k(\ell,s)
\end{eqnarray*}
where in the second last inequality we used the upper bound $G^{k-1}(\ell,s)\le \bound{(k-1)}$ derived in Lemma \ref{lem:upperBoundsgG}, and in the last inequality we used that for this case $G^k(\ell,s)=\Ginit{k}$ where $H b_0\ge H^*$.

We can now restrict to the case when $N_{\ell-1}(s)\ge 1$. 
For  such $s$, and $\ell\ge 2$, define
$$v^\ell(s) :=\frac{1}{N_{\ell-1}(s)}\sum_{t_i\in \text{ epoch }\ell-1: s_{t_i}=s} \frac{1}{H} \sum_{h=1}^H V^{t_i,h}(s)$$
Then, by algorithm construction, we have 
\begin{eqnarray}
    \Vbar^{\ell}(s) & = & \left\{\begin{array}{rl}
         [\spaProj v^{\ell}](s), & \text{ if } (\ell-1 \mod K)=0\\
         v^{\ell}(s), & \text{otherwise} 
    \end{array}
    \right.
\end{eqnarray}
By property $\spaProj v \le v$ (see Lemma \ref{lem: span projection} (b)), 
in either case $\Vbar^\ell\le v^\ell$. Then, since $s_{t_i}=s$, we can use the induction hypothesis $\is_{h}(\ell-1),h=1,\ldots, H$ to upper bound $V^{t_i,h}(s) = V^{t_i,h}(s_{t_i})$ in the above expression for $v^\ell$ to derive:
\begin{eqnarray*}
\Vbar^\ell(s) \le v^{\ell}(s) & \le & \frac{1}{N_{\ell-1}(s)}\sum_{t_i\in \ep \ell-1: s_{t_i}=s} \frac{1}{H} \sum_{h=1}^H \left([L^{H-h+1} \overline{L}^{k-1} \Vbar^{\ell-1-(k-1)}](s) + 4g^{k}(t_i,h)\right)\\
&=&\left[\overline{L}^{k} \Vbar^{\ell-k}\right](s)  + 4G^{k}(\ell,s).
\end{eqnarray*}
where the last equality follows from the definition of $\Lbar$ and the definition $G^k(\ell,s)$.
\end{proof}

In order to obtain the aggregate bound stated in Lemma \ref{lem:Q-upper-bound} from the per-step bound in Lemma \ref{lem: Q-upper-bound-per-step}, we need to bound the sum of $g^k(t,h)$ over $t,h$, which we derive next.
\label{apx: summation of g}


\begin{restatable}{lemma}{gksummation}
\label{corol: g k summation all epochs}
Given any  sequence of $0\le k_\ell \leq 2K-1$ for all epochs $\ell=1, 2,\ldots,$ with probability $1-\delta/2$,
  \begin{eqnarray}
     \sum\limits_{h=1}^H \sum\limits_{\ell, t\in \ep \ell}    g^{k_\ell+1}(t,h)  & \le & 2e KH  \sum\limits_{h=1}^H  \sum\limits_{t=1}^T  b_{n_{t,h}} + \gExtraTerm H, \nonumber\\
     & \le & C_1 \sum\limits_{t=1}^T  b_{n_{t,h}}+C_2
\end{eqnarray}
where $n_{t,h}=N^{t}(s_{t},a_{t,h}), a_{t,h}=\arg \max_a Q^{t,h}(s_{t},a)$, and
$\zeta=CS\log(T)$ denotes an upper bound on the number of epochs in Algorithm \ref{alg:main} and $\gCthreeSymb_K = \gCthree{K}$. And, $C_1=\Cone$, $C_2=\Ctwo$. 
\end{restatable}

To prove this lemma, We first prove the following lemma for bounding the summation over one epoch.

\begin{lemma}\label{lem: g k summation for one epoch}
For all epochs $\ell$, $h \in \{1,\ldots, H+1\}$ and $k\in\{1,\ldots, \ell-1\}$, the following holds with probability $1-\delta/2$, 
\begin{eqnarray}\label{eq: sum of g k}
    \sum_{t\in \ep \ell} g^{k}(t,h) 
    & \le &    \sum_{j=h}^{H} (1+\frac 1 C)^{j-h} \left(\sum_{t\in \mbox{ epoch } \ell} b_{n_{t,j}} + \gCthreeSymb\right) \nonumber \\ 
    && + (1+\frac 1 C)^{H-h+2} \left(\sum_{t\in\ep \ell-1} \frac{1}{H} \sum_{h'=1}^H g^{k-1}(t,h')  \right),
\end{eqnarray}
where $\gCthreeSymb=\gCthree{k}$.
\end{lemma}
\begin{proof}
We prove by induction on $h=H,H-1,\ldots 1$. First, we make some observations about the value of $\sum_{t\in \ep \ell} g^k(t,h)$ for any $h$.
By definition, for $t,h$ such that $n_{t,h}=N^t(s_t,a_{t,h})>0$,
\begin{eqnarray*}
g^{k}(t,h) & = & b_{n_{t,h}}+ \sum_{i=1}^{n_{t,h}} \alpha^i_{n_{t,h}} g^k(t_i+1,h+1),
\end{eqnarray*}
and for $n_{t,h}=N^t(s_t,a_{t,h})=0$, 
$$ g^{k}(t,h) = \ginit{(k-1)} 
\le \Ginit{k}.$$
Let $\tau_{\ell,h}$ be the number of steps in $\ep \ell$ for which $n_{t,h}:=N^t(s_t,a_{t,h})=0$. 
For a state-action pair $s,a$, consider the set of time steps at which $s_t=s, a_{t,h}=a$. Among these time steps, whenever $h_t=h$, we have $a_t=a_{t,h_t}=a_{t,h}=a$. Since $h_t$ is randomly sampled, once the number of such time steps (where $s_t=s, a_{t,h}=a$) exceeds  $H\log(1/\delta)$, we will have $h_t=h$ (and hence $a_t=a$) at least once with probability $1-\delta$. From there on, $N^t(s,a)\ge 1$. Therefore,
$$\tau_{\ell,h}=\sum_{t\in \ep \ell} I(N^t(s_t,a_{t,h})=0)=\sum_{s,a} \sum_{t\in \ell:s_t=s,a_{t,h}=a} I(N^t(s,a)=0)\le SAH\log(1/\delta)$$
w.p. $1-SA\delta$.
Substituting $\delta$ by $\frac{\delta}{2HSAT}$ we get that for all $\ell,h$ w.p. $1-\delta/2$
\begin{equation}
\label{eq:inittaubound}
    \tau_{\ell,h}\le SAH\log(2HSAT/\delta).
\end{equation}

For notational convenience let $\gCthreeSymb:=\gCthree{k}$. Now we provide our induction based proof for the lemma statement.


\noindent{\underline{Base Case $h=H$:}} By definition,  
\begingroup
\allowdisplaybreaks
\begin{eqnarray}\label{eq: sum of gk helper 1}
\sum_{t \in \ep  \ell} g^{k}(t,H) & \le & \sum_{t\in \ep \ell} b_{n_{t,H}}+ \sum_{t\in \ep \ell}  \sum_{i=1}^{n_{t,H}} \alpha^i_{n_{t,H}} G^{k}(\ell,s_{t_i+1}) + \Ginit{k}\tau_{\ell,h}\nonumber\\
& \le & \sum_{t\in \ep \ell} b_{n_{t,H}}+ \sum_{t:\ t+1\in \ep \ell} \left(\sum_{n\ge N^t(s_t,a_t)} \alpha_n^{N^t(s_t,a_t)}\right) G^{k}(\ell,s_{t+1})  \nonumber\\ 
&&+  \Ginit{k}\tau_{\ell,h}\nonumber\\
&& \text{(using that for all $i$, $\textstyle \sum_{n\ge i} \alpha^i_n \le 1+\frac{1}{C}$ from Lemma~\ref{lem: alpha summation properties} (c))} \nonumber\\
& \le & \sum_{t\in \ep \ell} b_{n_{t,H}}+ \sum_{t:\ t+1\in \ep \ell} (1+\frac 1 C)G^{k}(\ell,s_{t+1}) +\Ginit{k} \tau_{\ell,h}\nonumber\\
& \le & \sum_{t\in \ep \ell} b_{n_{t,H}}+ \sum_{t\in \ep \ell} (1+\frac 1 C)G^{k}(\ell,s_{t}) + \Ginit{k} \tau_{\ell,h}\nonumber\\
& \le & \sum_{t\in \ep \ell} b_{n_{t,H}}+  (1+\frac 1 C) \sum_{s} N^{\ell}(s)G^{k}(\ell,s)+ \Ginit{k} \tau_{\ell,h}\nonumber\\
& & \text{(note that for $s$ such that $N_{\ell-1}(s)=0$, by definition $G^{k}(\ell,s)=\Ginit{(k-1)}$)}\nonumber\\
& & \text{(also by epoch break condition $\sum_{s:N_{\ell-1}(s)=0} N_{\ell}(s)\le 1$)}\nonumber\\
&= & \sum_{t\in \ep \ell} b_{n_{t,H}}+  (1+\frac 1 C) \sum_{s:N^{\ell-1}(s)>0} \frac{N^{\ell}(s) }{N_{\ell-1}(s)}\sum_{t\in \ep \ell-1: s_{t}=s} \frac{1}{H} \sum_{h=1}^H g^{k-1}(t,h) \nonumber\\
& & +\Ginit{k} \tau_{\ell,h}+ \Ginit{(k-1)}\nonumber\\
&& \text{(using the epoch break condition in Algorithm~\ref{alg:main})}\nonumber\\
& \le & \sum_{t\in \ep \ell} b_{n_{t,H}}+  (1+\frac 1 C)^2 \sum_{s}\sum_{t\in \ep \ell-1: s_{t}=s} \frac{1}{H} \sum_{h=1}^H g^{k-1}(t,h) \nonumber\\ && +\Ginit{k} (\tau_{\ell,h}+1)\nonumber\\
& \le & \sum_{t\in \ep \ell} b_{n_{t,H}}+  (1+\frac 1 C)^2 \sum_{t\in \ep \ell-1} \frac{1}{H} \sum_{h=1}^H g^{k-1}(t,h)+\gCthreeSymb,\nonumber
\end{eqnarray}
\endgroup
With probability $1-\delta$, Here in the last inequality we used the high probability upper bound  on $\tau_{\ell,h}$ derived in \eqref{eq:inittaubound}.

\paragraph{\underline{\it Induction step:}} Assume \eqref{eq: sum of g k} holds for $h+1$. We show that it holds for $h$. 
\begingroup
\allowdisplaybreaks
\begin{eqnarray}
\sum_{t\in \ep \ell} g^k(t,h) & \le & \sum_{t\in \ep \ell} b_{n_{t,h}} + \sum_{t\in \ep \ell} \sum_{i=1}^{n_{t,h}} \alpha^i_{n_{t,h}} g^k(t_i+1,h+1)+ \tau_{\ell,h} (\Ginit{k})  \nonumber\\
& \le &    \sum_{t\in \ep \ell} b_{n_{t,h}} + \sum_{(t+1)\in \ep \ell} \sum_{n\ge N^t(s_t,a_t)} \alpha^{N^t(s_t,a_t)}_n g^k(t+1,h+1) +\gCthreeSymb\nonumber\\
&& \text{(Lemma~\ref{lem: alpha summation properties} (c) gives)}\nonumber\\
    & \le &  \sum_{t\in \ep \ell} b_{n_{t,h}} + \sum_{(t+1)\in \ep \ell} (1+\frac{1}{C})g^{k}(t+1,h+1)  + \gCthreeSymb \nonumber\\
    & \le &  \sum_{t\in \ep \ell} b_{n_{t,h}} + \sum_{t\in \ep \ell} (1+\frac{1}{C})g^{k}(t,h+1)  +  \gCthreeSymb\nonumber\\
     &  & \text{ (applying induction hypothesis for $h+1$)} \nonumber\\
      & \le & \sum_{t\in \ep \ell} b_{n_{t,h}}+ (1+\frac{1}{C})   \sum_{j=h+1}^H (1+\frac 1 C)^{j-(h+1)} (\gCthreeSymb +\sum_{t\in \ep \ell} b_{n_{t,j}}) \nonumber\\
      & & + (1+\frac{1}{C})  (1+\frac 1 C)^{H-h+1} \sum_{t\in \ep \ell-1} \frac{1}{H} \sum_{h'=1}^H g^{k-1}(t,h') + \gCthreeSymb \nonumber\\
       & = &  \sum_{j=h}^H { (1+\frac 1 C)^{j-h}} (\gCthreeSymb +\sum_{t\in \mbox{ epoch } \ell} b_{n_{t,j}}) \nonumber \\  &&+ { (1+\frac 1 C)^{H-h+2}}  \sum_{t\in \ep \ell-1} \frac{1}{H} \sum_{h'=1}^H g^{k-1}(t,h').  \nonumber
\end{eqnarray}
\endgroup

\end{proof}

\begin{corollary}
\label{corol: g k summation}
For all epochs $\ell$ and all $k\in \{1,\ldots, \ell-1\}$, the following holds  with probability at least $1-\delta/2$,
    \begin{eqnarray*}
    \frac{1}{H} \sum_{h=1}^H \sum_{t\in \mbox{ epoch } \ell} g^{k}(t,h) 
    & \le & (1+\frac 1 C)^{k(H+2)}  \sum_{h=1}^H \left(\gCthreeSymb_k k+ \sum_{t\in \mbox{ epoch } \ell \ldots, \ell-k+1} b_{n_{t,h}} \right),
\end{eqnarray*}
where $\gCthreeSymb_k=\gCthree{k}$.
\end{corollary}
\begin{proof}
Lemma~\ref{lem: g k summation for one epoch} gives
\begin{eqnarray*}
    \frac{1}{H} \sum_{h=1}^H \sum_{t\in \mbox{ epoch } \ell} g^{k}(t,h) 
    & \le &   (1+\frac 1 C)^{H} \sum_{h=1}^H\left(\gCthreeSymb+\sum_{t\in \mbox{ epoch } \ell}   b_{n_{t,h}} \right)\nonumber\\
    & & + (1+\frac 1 C)^{H+2} (\sum_{t\in \ep \ell-1} \frac{1}{H} \sum_{h'=1}^H g^{k-1}(t,h') ) \nonumber\\
\end{eqnarray*}
Upon recursively applying the above for $k-1,  k-2, \ldots, 1$ and using that $g^0(t,h)=0$ for all $t,h$, we get the required statement. We also used that $\gCthreeSymb_{k'}$ in the inequalities for $k'=k-1,k-2,\ldots,$ is upper bounded by $\gCthreeSymb_k$.
\end{proof}

Lemma \ref{corol: g k summation all epochs} then follows as a corollary of above results.

\begin{proof}[Proof of Lemma \ref{corol: g k summation all epochs}]
We sum the statement in Corollary~\ref{corol: g k summation} over all $\ell$ while substituting $k$ by $k_\ell+1 \le 2K$ in the statement for epoch $\ell$. Then, note that in the summation in the right hand side, the term ($\sum_{t\in \ep \ell}   \sum_{h=1}^H b_{n_{t,h}}$) for each epoch $\ell$ can appear multiple times (in the corresponding terms for some $\ell' \ge \ell$ such that $\ell\ge \ell'-k_{\ell'}$). Indeed, because $k_{\ell'}\le 2K-1$ for all $\ell'$, each epoch term can appear in at most $2K$ terms which gives a factor $2K$, and we get
    \begin{eqnarray*}
     \frac{1}{H} \sum_\ell \sum_{t\in \ep \ell}  \sum_{h=1}^H  g^{k_\ell+1}(t,h)  & \le & (1+\frac{1}{C})^{2K (H+2)}  \sum_{h=1}^H \left( 2K\sum_\ell \sum_{t\in \ep \ell}  b_{n_{t,h}} + 2K\zeta\gCthreeSymb_{2K}\right)
\end{eqnarray*} 
where $\zeta$ is an upper bound on the number of epochs. 
   Finally, we use that $(1+\frac 1 C)^{2K (H+2)} = (1+\frac 1 C)^{C} \le e$ using $C=\constC$ to get the lemma statement. 
\end{proof}

Now we are ready to prove Lemma \ref{lem:Q-upper-bound}.

\QupperBound*
\begin{proof}
    This result is obtained by essentially summing the per-step upper bound provided by Lemma \ref{lem: Q-upper-bound-per-step} over $t,h$. Specifically, with probability $1-\delta/2$, \eqref{eq:induction4} holds for every $t$, for $h=1,\ldots, H$ and $k_t$ such that $\ell_t-k_t =Kj+1$. Summing over $t=1,\ldots, T$ and $h=1,\ldots, H$, we get
    $$ \sum_{t,h} Q^{t,h}(s_t,a_{t,h}) \le  \sum_{t,h} \left(R(s_t,a_{t,h}) + P_{s_t,a_t} \cdot L^{H-h}\overline{L}^{k} \Vbar^{\ell-k}\right) + 4 \sum_{t,h} g^{k}(t,h).$$
    Then, we substitute the bound derived in Lemma \ref{corol: g k summation all epochs} to get  $4\sum_{t,h} g^{k}(t,h) \le C_1 \sum_{t,h} b_{n_{t,h}} + C_2$ for $C_1=\Cone, C_2=\Ctwo$ . Since this bound holds with probability $1-\delta/2$, we get the lemma statement with probability $1-\delta$. 
\end{proof}


\subsection{Proof of Theorem~\ref{thm: main regret}}\label{apx: proof of regret bound}
\label{app:regret-analysis}
In this section, we expand upon the proof sketch in Section \ref{sec: regret analysis} to provide a complete proof of Theorem \ref{thm: main regret}. We restate the theorem for ease of reference.
\thmMain*
\begin{proof} 
Combining the upper bound on $\sum_{t,h} V^{t,h}(s_t)$ given by Lemma \ref{lem:Q-upper-bound} with the lower bound on each $V^{t,h}(s_t), \forall t,h$ given by the optimism lemma (Lemma \ref{lem:optimismNew}) for some appropriate $k=k_t$, which is defined later. Therefore, we have that with probability at least $1-2\delta$,
\begin{equation} 
 \sum_{t,h} L^{H-h+1} \Lbar^{k_t} \Vbar^{\ell_t-k_t} (s_t) \le \sum_{t,h} \left( R(s_t,a_{t,h}) + P_{s_t,a_{t}} \cdot L^{H-h}\overline{L}^{k_t} \Vbar^{\ell_t-k_t} \right)+ C_1   \sum_{t,h} b_{n_{t,h}} + C_2.
\end{equation} 
For notational convenience, denote $v^{t,h}:=L^{H-h}\Lbar^{k_t} \Vbar^{\ell_t-k_t}$.
Then, by moving the terms around in above  above we get, 
\begin{eqnarray}
\label{eq:rewardLowerbound}
 \sum_{t,h} R(s_t,a_{t,h})  & \ge & \sum_{t,h} \left( [Lv^{t,h}](s_t) -P_{s_t,a_t} v^{t,h} \right)- C_1   \sum_{t,h} b_{n_{t,h}} - C_2,
\end{eqnarray} 
Now, the optimal asymptotic average reward $\rho^*$ for  average reward MDP (see Section~\ref{sec: setting}) satisfies the following Bellman equation for any state $s_t$: 
\begin{eqnarray}
\label{eq:Bellman}
 \rho^* & = & [LV^*](s_t) - V^*(s_t) \nonumber\\
       & = & [LV^*](s_t) - P_{s_t,a_t} V^* + (P_{s_t,a_t} V^*-V^*(s_t)).
\end{eqnarray}
 Summing \eqref{eq:Bellman} over $t,h$ and subtracting \eqref{eq:rewardLowerbound}, we get the following bound:
\begin{eqnarray}
\label{eq:expectedRegretEquationApp}
HT\rho^* -  \sum_{t,h} R(s_t,a_{t,h})  & \le & \sum_{t,h} \underbrace{\left([LV^*](s_t) - [Lv^{t,h}](s_t) - P_{s_t,a_t}(V^*-v^{t,h})\right)}_{\text{Term 1}}\nonumber\\
&&+ \underbrace{\sum_{t,h} (P_{s_t,a_t}V^*-V^*(s_t))}_{\text{Term 2}} +  C_1   \underbrace{\sum_{t,h} b_{n_{t,h}}}_{\text{Term 3}} + C_2.
\end{eqnarray}
We now bound the three terms in the above inequality. 

\paragraph{Bounding Term 1.} To bound the first term in the above, observe that for every $t$ (by definition of the L-operator),
\begin{eqnarray*}
    LV^*(s_t) - Lv^{t,h}(s_t) & = & \max_a \left(R(s_t,a)+ P(s_t,a)V^*\right) - \max_a \left(R(s_t,a)+ P(s_t,a)v^{t,h}\right)\\
    & \le &  \max_a \left(R(s_t,a)+ P(s_t,a)V^* - R(s_t,a) -  P(s_t,a)v^{t,h}\right)\\
    & = & \max_a P(s_t,a)(V^* - v^{t,h}).
\end{eqnarray*}
Therefore, for each $t,h$, by definition of span,
$$LV^*(s_t) - Lv^{t,h}(s_t) - P_{s_t,a_t}(V^*-v^{t,h}) \le \max_a (P_{s_t,a}-P_{s_t,a_t})(V^*-v^{t,h}) \le \spa(V^*-v^{t,h}),$$
Then, since $\spa(V^*-L^{H-h}\Lbar^{k_t}V^*)=0$, we can use the span contraction property of the standard Bellman operator $L$ (discount factor $1$) to derive  
$$
    \spa(V^*-v^{t,h}) = \spa(L^{H-h}\Lbar^{k_t}V^*-L^{H-h}\Lbar^{k_t} \Vbar^{\ell_t-k_t}) \le \spa(\Lbar^{k_t}V^*-\Lbar^{k_t} \Vbar^{\ell_t-k_t}).
$$
Then, using the strict contraction property of $\Lbar$ derived in Lemma \ref{lem: span contraction property}, we obtain 
$$ \spa(V^*-v^{t,h}) \le   (1-\frac{p}{H})^{k_t} \spa(V^*-\Vbar^{\ell_t-k_t}).$$
Therefore, we have the following bound on term 1:
$$\text{Term 1} \le \sum_t (1-\frac{p}{H})^{k_t} \spa(V^*-\Vbar^{\ell_t-k_t}),$$
where recall that as per the conditions in Lemma \ref{lem:optimismNew} and Lemma \ref{lem:Q-upper-bound}, for any $t\in \ell$, $k_t$ is of the form $\ell-(Kj+1)$ for certain feasible $j$. Specifically, we set $j=\jvalue$ which satisfies the conditions in both the lemmas. 
This gives
$$k_t := \left\{ \begin{array}{ll}
\ell-1, & \text{ for } \ell \le 2K\\
\ell-(K\lfloor \frac{\ell-K-1}{K} \rfloor+1), & \text{ for } \ell \ge 2K+1
\end{array}\right.
$$
Thus, we have $k_t \in [K, 2K-1]$ for all $t$ in epochs $\ge K+1$. Substituting, 
$$\sum_{t} (1-\frac{p}{H})^{k_t} \leq \sum_{\ell=1}^K \tau_\ell (1-\frac{p}{H})^{\ell-1}+ \sum_{\ell\ge K+1} \tau_\ell (1-\frac{p}{H})^{K}$$
where $\tau_\ell$ is the length of epoch $\ell$.
Now, due to the epoch break condition we have $\tau_\ell\le (1+\frac{1}{C}) \tau_{\ell-1}\le (1+\frac{p}{H}) \tau_{\ell-1}$, so that $\sum_{\ell=1}^K \tau_\ell (1-\frac{p}{H})^{\ell-1} \le (1-\frac{p^2}{H^2})^{\ell-1} \le \frac{H^2}{p^2}$. Also, using $K=\constK$, we have $(1-p/H)^K \leq 1/T^2$, so that $\sum_{\ell\ge K+1} \tau_\ell (1-\frac{p}{H})^{K} \sum_{t=1}^{T} T\frac{1}{T^2} \le \frac{1}{T}$. Combining these observations, 
$$\sum_{t} (1-\frac{p}{H})^{k_t} \le \frac{H^2}{p^2} + \frac{1}{T}.$$
Further, since $\ell_t-k_t=Kj+1$ for some $j\ge 0$, due to the projection step that we execute in the algorithm at the epochs of form $Kj$, we have $\spa(\Vbar^{\ell_t-k_t}) = \spa(\Vbar^{Kj+1})\le 2H^*$ so that $\spa(V^*-\Vbar^{\ell_t-k_t})\le 2H^*+\spa(V^*) \le 3H^*\le 3\frac{H}{p}$.

To conclude, with probability $1-2\delta$, the first term in \eqref{eq:expectedRegretEquationApp} is bounded as
$$\text{Term 1} \le \sum_t (1-\frac{p}{H})^{k_t} \spa(V^*-\Vbar^{\ell_t-k_t}) \le \frac{3H^3}{p^3} + \frac{3H}{pT}$$

\paragraph{Bounding Term 2.}
To bound the second term, note that $\Ex[V^*(s_{t+1})|s_t,a_t]=P_{s_t,a_t}V^*$, so that $\sum_{t=1}^{T-1} (P_{s_t,a_t} V^*-V^*(s_{t+1}))$ is a martingale sum where the absolute value of each term $P_{s_t,a_t} V^*-V^*(s_{t+1})$ is bounded by $\spa(V^*)\le H^*$. Therefore, applying Azuma-Hoeffding inequality (see Lemma \ref{lem: Azuma hoeffding}, it is bounded by $O(H^*\sqrt{T \log(1/\delta)})$ with probability $1-\delta$. This bounds $\sum_t (P_{s_t,a_t} V^*-V^*(s_t)) = \sum_{t=1}^{T-1} (P_{s_t,a_t} V^*-V^*(s_{t+1}))+V^*(s_T)-V^*(s_1)\le O(H^*\sqrt{T \log(1/\delta)})$ with probability $1-\delta$.

\paragraph{Bounding Term 3.}
The leading contributor to our regret bound is the third term in \eqref{eq:expectedRegretEquationApp} that requires bounding the sum of all bonuses $\sum_{t,h}  b_{n_{t,h}}$ where $b_{n_{t,h}}=\alphaBonus{n_{t,h}}$.  

Recall that at any time $t$, action $a_t=a_{t,h_t}$, the arg max action of $Q^{t,h_t}(s_t,\dot)$ where $h_t$ was sampled uniformly at random from $\{1,\ldots, H\}$. Therefore, for any $h$, with probability $1/H$, $n_{t,h}:=N^t(s_t,a_{t,h})=N^t(s_t,a_{t,h_t}) = N^t(s_t,a_t)=:n_t$ . And,
$$
\frac{1}{H} \sum_{t=1}^T \sum_{h=1}^H \frac{1}{\sqrt{n_{t,h}+1}} = \sum_{t=1}^T \Ex[\frac{1}{\sqrt{n_t+1}}] =  \sum_\ell \sum_{s,a} \sum_{i=1}^{N_\ell(s,a)+1} \frac{1}{\sqrt{i}} = \sum_\ell \sum_{s,a} \sqrt{N_\ell(s,a)+1}
$$
where  $N_{\ell}(s,a)$ denotes the number of visits to $(s,a)$ by the end of the $\ell_{th}$ epoch. Then, using that $\sum_\ell \sum_{s,a} N_\ell(s,a)=T$ and Cauchy-Schwartz inequality,
$$\sum_\ell \sum_{s,a} \sqrt{N_\ell(s,a)+1} \le  \sqrt{\zeta SA(T+\zeta SA)}$$
with $\zeta=CS\log(T)$ being an upper bound on the total number of epochs.
Using $H^*\leq \frac{2H}{p}$ from Lemma~\ref{lem:spanVstar}, and substituting back we get,
\begin{eqnarray*}
    \sum_{t=1}^T \sum_{h=1}^H b_{n_{t,h}} 
    & \le &  24 HH^*\sqrt{C\log(8 SAT/\delta)}  (\sqrt{\zeta SAT} + \zeta SA) \\
    & = &  O\del{\frac{1}{p^2}H^4 S\sqrt{AT\log(SAT/\delta)}\log(T)+ \frac{1}{p^{2.5}}H^{5} S^2A \sqrt{\log(SAT/\delta)}\log(T)^{1.5}}.
\end{eqnarray*}

\paragraph{Final regret bound.} On substituting the derived bounds on the three terms into \eqref{eq:expectedRegretEquationApp} along with $C_1=\Cone,C_2=\Ctwo$, and dividing the entire equation by $H$, we have that with probability $1-3\delta$,
\begin{equation}
\label{eq:reg1}
T\rho^* -  \frac{1}{H} \sum_{t,h} R(s_t,a_{t,h}) \le \regBoundO
\end{equation}
To connect this bound to the definition of regret, recall that given state $s_t$ at time $t$, the algorithm samples $h_t$ uniformly at random from $\{1,\ldots, H\}$ and takes action $a_t$ as arg max of $Q^{t,h_t}(s_t,\cdot)$. Therefore, $a_t=a_{t,h_t}$, and we have
 $$\Ex[\Reg(T)] = T\rho^* - \Ex[\sum_t R(s_t,a_{t,h_t})] = T\rho^* - \Ex[\frac{1}{H} \sum_{t,h} R(s_t,a_{t,h})]$$
 Therefore, the upper bound in  \eqref{eq:reg1} gives a bound on the expected regret of the algorithm. Since the per-step reward is bounded by $1$, the high probability regret bound of Theorem \ref{thm: main regret} is then obtained by a simple application of Azuma-Hoeffding inequality. This introduces an extra $O(\sqrt{T\log(1/\delta)})$ term which does not change the order of regret.
\end{proof}


\section{Supporting lemma}
\label{apx: supporting lemma}

\subsection{Properties of Projection Operator}
\begin{definition}[Projection]
Define operator $\spaProj:\mathbb{R}^S \rightarrow \mathbb{R}^S$ as: for any $v\in \mathbb{R}^S$
$$[\spaProj v](s) := \min\{2H^*, v(s)-\min_{s\in \calS} v(s)\} + \min_{s\in \calS} v(s).$$
\end{definition}
We show that this operator satisfies the following properties, which will be useful in our analysis later. 
\begin{lemma}
\label{lem: span projection}
For any vector $v \in \mathbb{R}^S$,
\begin{enumerate}
    \item[(a)] $\spa(\spaProj v)\le 2H^*$,
    \item[(b)] $\spaProj v \le  v$, and
    \item[(c)] for any vector $u\le v$, $\spaProj u \le \spaProj v$.
    \item[(d)] for any $v$ with $\spa(v)\le 2H^*$, $\spaProj v = v$.
\end{enumerate}
\end{lemma}


\begin{proof}[Proof of (a),(b), and (d)] These statements are trivially true by definition of the $\spaProj$ operator.
\end{proof}
\begin{proof}[Proof of (c).] 
Fix an $s$. For any vector $v$, we say that $\spaProj v(s)$ is `clipped' if $\spaProj v(s) \ne v(s)$, i.e., iff $v(s) - \min_{s'} v(s') > 2H^*$, so that $\spaProj v(s) =  2H^* +  \min_{s'} v(s') < v(s)$. We compare $\spaProj u(s),\spaProj v(s)$ by analyzing all possible four cases:
\begin{itemize}
    \item If both $\spaProj u(s) \, \&\,\spaProj v(s)$ are not clipped, then $\spaProj u(s) = u(s) \leq \spaProj v(s)=v(s) $.
    \item If both $\spaProj u(s) \, \&\,\spaProj v(s)$ are clipped, then $\spaProj u(s) = 2H^*+\min_s u(s) \leq 2H+\min_{s} v(s) = \spaProj v(s) $.
    \item If $\spaProj u(s)$ is clipped but $\spaProj v(s)$ is not clipped, i.e. $\spaProj u(s) < u(s)$ but $\spaProj v(s)=v(s)$ then clearly $\spaProj u(s) < u(s) \leq v(s) = \spaProj v(s) $.
    \item Finally, if $\spaProj v(s)$ is clipped, i.e. $\spaProj v(s) < v(s)$ but $\spaProj u(s)=u(s)$, then we have $u(s) \leq 2H^*+\min_s u(s)$ and $\spaProj v(s) = 2H^* +\min_s v(s)$. Therefore, $\spaProj v(s) \geq \spaProj u(s)$.
\end{itemize}   
\end{proof}
\subsection{Span bounds for $\Vbar^{\ell}$}
In this subsection, we prove some span bounds for vector $\Vbar^{\ell}$, for any $\ell$ of form $\ell=Kj+1$ for some integer $j$. 

\begin{lemma}\label{lem: trivial span bound KH}
For any integer $j\ge 0$, and any $k\ge 0$,
    $$
    \spa\del{L^{H-h+1}\overline{L}^{k}\Vbar^{Kj+1}} \leq 4H^*,\quad h\in \{1, \ldots, H+1\}.
    $$
And,  for $k\ge K$, we have
 $$
    \spa\del{L^{H-h+1}\overline{L}^{k}\Vbar^{Kj+1}} \leq 2H^* ,\quad h\in \{1, \ldots, H+1\}.
    $$
\end{lemma}
\begin{proof}
Fix an $h \in 1, \ldots,H+1$. Consider,
\begin{eqnarray*}
    \spa\del{L^{H-h+1}\overline{L}^{k}V^{Kj+1,H+1}} &\leq & \spa\del{L^{H-h+1}\overline{L}^{k}V^{Kj+1,H+1}-V^*}+ \spa\del{V^*} \\
    && \text{(from Lemma~\ref{lem: span contraction property})}\\
    &\leq& \del{1-p/H}^{k}\spa\del{V^{Kj+1,H+1}-V^*} + H^*\\
    & \le & \spa\del{V^{Kj+1,H+1}-V^*} + H^*\\
    &\leq& 4H^*,
\end{eqnarray*}
where in the last inequality we 
used that either from initialization ($j=0$) or from projection operation ($j\geq 1$), we have  $\spa(V^{Kj+1}) \le 2H^*$, so that $\spa(V^{Kj+1,H+1}-V^*) \le 3H^*$.
For $k\ge K$, we can extend above to obtain
\begin{eqnarray*}
    \spa\del{L^{H-h+1}\overline{L}^{k}V^{Kj+1,H+1}}     &\leq& \del{1-p/H}^{k}\spa\del{V^{Kj+1,H+1}-V^*} + H^*\\
    &\leq& \del{1-p/H}^{K} (3H^*) + H^*\\
    &\leq& 2H^*,
\end{eqnarray*}
where in the last inequality we used that $K=\constK$ (assuming $T \ge 3 H^*$).
\end{proof}

\section{Technical preliminaries}
\begin{lemma}[Lemma 4.1 in~\cite{jin2018q}]\label{lem: alpha summation properties}
The following holds:
\begin{enumerate}[label=(\alph*)]
    \item $
    \frac{1}{\sqrt{n}}\leq \sum_{i=1}^n \frac{\alpha^i_n}{\sqrt{i}} \leq \frac{2}{\sqrt{n}}$. 
    \item $\max_{i\in n}\alpha_n^i \leq \frac{2C}{n}$ and $\sum_{i=1}^n(\alpha_n^i)^2 \leq \frac{2C}{n}.$
    \item $\sum_{n=i}^{\infty}\alpha^i_n \leq 1+1/C$.
\end{enumerate}
\end{lemma}

\begin{lemma}[Martingale Concentration, Corollary 2.20 in~\cite{wainwright2019high}]\label{lem: Azuma hoeffding}
		Let $(\{A_i,\mathcal{F}_i\}_{i=1}^{\infty})$ be a martingale difference sequence, and suppose $|A_i|\leq d_i$ almost surely for all $i\geq 1$. Then for all $\eta\geq 0$,
		\begin{equation}
		\mathbb{P}\left[|\sum_{i=1}^nA_i| \geq \eta\right] \leq 2 \exp\del{\frac{-2\eta^2}{\sum_{i=1}^n d_i^2}}.
		\end{equation}
		In other words, with probability at least $1-\delta$, we have,
		\begin{equation}
		|\sum_{i=1}^nA_i| \le \sqrt{\frac{\ln\del{2/\delta}\sum_{i=1}^nd_i^2}{2}}
		\end{equation}
	\end{lemma}

	

\begin{lemma}\label{lem: concentration based on Chi jin}
Given an $s,a$, let ${\cal G}_n = (r_i, s_{i+1}), i=1,\ldots, n$ be the (reward, next-state) for the first $n$ occurrences of $s,a$ in  Algorithm~\ref{alg:main}. Then, given fixed vector $V\in \R^S$ with $\spa (V) \leq \sigma$ 
the following holds with probability at least $1-\delta$, for all $n,s,a$,
 \begin{eqnarray*}
\left|\sum_{i=1}^{n}\alpha_n^i\del{r_i+ V(s_{i+1})-(R(s_i,a_i)+P_{s_i,a_i}\cdot V)}\right| & \leq & (\sigma+1)\sqrt{\frac{2C\log(SAT/\delta)}{n}} 
\end{eqnarray*}
where $s_i=s,a_i=a$. Further, if  $\sigma\le 4H^*$, then with probability at least $1-\frac{\delta}{8HT^3}$, for all $n,s,a$,
\begin{eqnarray*}
\left|\sum_{i=1}^{n}\alpha_n^i\del{r_i+ V(s_{i+1})-(R(s_i,a_i)+P_{s_i,a_i}\cdot V)}\right|&\leq& \alphaBonus{n} =: b_n 
\end{eqnarray*}
\end{lemma}
\begin{proof}
Define $x_i =\alpha_n^i( r_i-R(s,a)+V(s_{i+1})-P_{s_i,a_i}\cdot V)$. 
Then, $x_i$ is a margtingale difference sequence. 
Since $|V(s_{i+1})-P_{s_i,a_i}\cdot V|\leq \spa (V), |r_i-R(s,a)|\le 1$, using Lemma~\ref{lem: alpha summation properties} (b), we have $\sum_{i}^n x_i^2\leq 2C{(\sigma+1)}^2/n$. Then, applying the martingale concentration bound from Lemma~\ref{lem: Azuma hoeffding} combined with an union bound over all $(s,a)\in \calS\times\calA$ and all possible values of $n$, we get the following with probability at least $1-\delta$, for all $s,a,n$
\begin{eqnarray*}
\sum_{i=1}^{n} x_i & \leq  & (\sigma+1)\sqrt{\frac{2C\log(SAT/\delta)}{n}}.  
\end{eqnarray*}
The second inequality is a simple corollary of the first inequality in the lemma statement by substituting $\sigma$ by $4H^*$, $\delta$ by $\frac{\delta}{8HT^3}$ and using the observation that $\frac{1}{n+1} \geq \frac{1}{2n}$.
\end{proof}

\section{PAC guarantee}
\label{apx: PAC}
In addition to regret guarantees, our results imply that Algorithm \ref{alg:main} is also a PAC-learning algorithm. For PAC-guarantee of $(\epsilon,\delta)$, we seek a policy $\pi$ such that with probability $1-\delta$, the policy is $\epsilon$-optimal in the sense that $\rho^* - \rho^\pi \leq \epsilon$. 

We show that we can use Algorithm~\ref{alg:main} to construct a policy with $(\epsilon, \delta)$-PAC guarantee using $T$ samples where $\epsilon = \frac{3\Reg(T)}{T} + O(H \sqrt{\frac{\zeta}{T}\log(1/\delta)})$, with $\zeta$ being the number of epochs in the algorithm. Substituting $\zeta=O(H^2 S\log(T))$ and the $\tilde O(H^5S\sqrt{AT})$  regret bound from Theorem~\ref{thm: main regret general}, this provides a way to get $(\epsilon,\delta)$-PAC policy using $\tilde O(\frac{H^{10}S^2AT}{\epsilon^2})$ samples.

The desired policy can simply be constructed at the end of $T$ time steps by picking one out of the $T$ policies $\pi_1,\ldots, \pi_T$ used by the algorithm uniformly at random. As we show in Lemma \ref{lem: pac result} in the appendix, such a policy $\overline{\pi}$  is $\epsilon$-optimal  with probability at least $2/3$. Then, repeating this experiment $3 \log(1/\delta)$ times and picking the best policy (by estimating $\rho^\pi$ of each policy, which by Lemma 3 from \cite{agrawal2022learning} can be done efficiently under Assumption \ref{assume: expected hitting time assumption}), we can obtain the desired $(\epsilon,\delta)$-PAC-guarantee.

\begin{lemma}[PAC guarantee]\label{lem: pac result}
Let $\pi_1,\ldots \pi_T$ denote the policy used by Algorithm \ref{alg:main} at time step $t=1,\ldots, T$. Consider the policy 
$\overline{\pi}$ constructed by picking one of these $T$ policies uniformly at random. That is, $\overline{\pi} \sim \texttt{Uniform}\{\pi_1,\pi_2,\ldots,\pi_T\}$. 
Then, with probability at least $2/3$, $\rho^* - \rho^{\overline{\pi}} \leq \epsilon$, where  $\epsilon = \frac{\Reg(T)}{T}+O(H\sqrt{\frac{\zeta T\log(1/\delta)}{T}})$,
and $\zeta=CS\log(T)$ denotes an upper bound on the number of epochs in Algorithm \ref{alg:main}.
\end{lemma}
\begin{proof}
By Lemma \ref{lem: unichain} and Section 8.3.3 in \cite{puterman2014markov}, we have that $\rho^{\overline{\pi}}(s_1) = \rho^{\overline{\pi}}(s_2) =: \rho^{\overline{\pi}}$ for all $s_1,s_2$. To prove that $\overline{\pi}$ is $\epsilon$-optimal with probability $2/3$, we first prove that for any state $s$,  $\Ex[\rho^{\overline{\pi}}] \ge \rho^*-\frac{\epsilon}{3}$. Then, by Markov inequality $\Pr(\rho^*-\rho^{\overline{\pi}}(s) \ge \epsilon) \le \frac{1}{3}$, and we get the desired result.

Now, observe that by construction of Algorithm \ref{alg:main}, we have  
$$\Ex[\rho^{\overline{\pi}}] = \frac{1}{T} \sum_{t=1}^T \rho^{\pi_t} = \frac{1}{T} \sum_{\ell=1}^L \tau_\ell \rho^{\pi_\ell},$$
where $\pi_\ell$ denotes the (stationary) policy used by the algorithm in epoch $\ell$. 
In Lemma \ref{lem: boundedBiasAllPolicies}, we show that for every stationary policy the span of bias vector is bounded by $2H$ under Assumption \ref{assume: expected hitting time assumption}. Therefore, applying the result from \cite{agrawal2022learning} (Lemma 3, restated in our notation as Lemma \ref{lem: conc from inventory management}) that provides a bound on the difference between asymptotic average and empirical average reward of any stationary policy that has bounded bias,  we get that with probability $1-\delta$,
$$\bigg|\tau_\ell \rho^{\pi_\ell} - \sum_{t\in \ep \ell} R(s_t, a_t)\bigg| \le  O(H\sqrt{\tau_\ell \log(1/\delta)})$$
Summing over (at most) $\zeta$ epochs, and substituting back in the expression for $\Ex[\rho^{\overline{\pi}}]$, we get 
$$T \Ex[\rho^{\overline{\pi}}] \ge \sum_{t=1}^T R(s_t, a_t) -  O(H\sqrt{\zeta T \log(1/\delta)}) $$
Then, substituting the definition of regret $\Reg(T)=T\rho^* - \sum_{t=1}^T R(s_t,a_t)$, we get the desired bound on $\Ex[\rho^{\overline{\pi}}]$:
$$\Ex[\rho^{\overline{\pi}}] \ge  \rho^* - \frac{\Reg(T)}{T} - O(H\sqrt{\frac{\zeta T\log(1/\delta)}{T} }) $$
\end{proof}

\begin{lemma}[Restatement of Lemma 3 of~\cite{agrawal2022learning}]\label{lem: conc from inventory management}
Given any policy $\pi$ with bias bounded by $H$, let $s_1,a_1,\ldots, s_\tau, a_\tau$ denote the states visited and actions taken at time steps $k=1,\ldots, \tau$ on running the policy starting from state $s_1$ for time $\tau$ which is a stopping time relative to filtration ${\cal F}_k=\{s_1,a_1, \ldots, s_k, a_k\}$. Then,
\begin{equation}
    |\tau\rho^\pi - \sum_{k=1}^\tau R(s_k,a_k)| \leq 2H\sqrt{2\tau\log(2/\delta)} 
\end{equation}
\end{lemma}

\end{appendix}

\end{document}